\documentclass{article}
\PassOptionsToPackage{numbers, sort, compress}{natbib}
\usepackage[final]{neurips_2026}
\usepackage[utf8]{inputenc} 
\usepackage[T1]{fontenc}
\usepackage{booktabs}
\usepackage[table]{xcolor}
\usepackage{graphicx}

\usepackage{graphicx}
\usepackage{booktabs}
\usepackage{subcaption}
\usepackage{amssymb}
\usepackage{amsmath}
\usepackage{amsthm}
\usepackage{amsfonts}
\usepackage{mathtools}
\usepackage{xcolor}
\usepackage{enumitem}
\usepackage{xspace}
\usepackage{microtype}
\usepackage{bm}
\usepackage{algorithm}
\usepackage{algorithmic}
\usepackage{listings}
\usepackage{etoc}
\usepackage{authblk}

\usepackage{xurl}
\definecolor{bluegray}{rgb}{0.4, 0.6, 0.8}
\definecolor{lightcarminepink}{rgb}{0.9, 0.4, 0.38}
\usepackage[pagebackref,breaklinks,colorlinks,citecolor=brown,linkcolor=lightcarminepink]{hyperref}

\definecolor{funcBlue}{RGB}{43,145,175} 
\definecolor{commentTeal}{RGB}{110,154,155}

\DeclareMathOperator*{\supp}{supp}

\renewcommand{\aa}{\mathbf{a}}

\let\lll\ll
\renewcommand{\ll}{\mathbf{l}}

\providecommand{\xx}{\mathbf{x}}

\providecommand{\zz}{\mathbf{z}}

\providecommand{\cS}{\mathcal{S}}

\providecommand{\cU}{\mathcal{U}}

\providecommand{\R}{\mathbb{R}} 

\usepackage{mathtools}

\newcommand{\fixedxrightarrow}[1]{%
  \xrightarrow{\mathmakebox[2.0em][c]{#1}}%
}

\definecolor{fred}{HTML}{D62727}

\newcommand{\eg}{\textit{e.g.}}
\newcommand{\ie}{\textit{i.e.}}

\newcommand{\E}{\mathbb{E}}
\newcommand{\Id}{\mathbf{I}}
\newcommand{\op}{\mathrm{op}} 

\newcommand{\methodname}{World Action Verifier\xspace}
\newcommand{\methodacro}{WAV\xspace}

\usepackage{titlesec}
\titlespacing\section{0pt}{4pt plus 2pt minus 2pt}{0pt plus 2pt minus 2pt}
\titlespacing\subsection{0pt}{2pt plus 2pt minus 2pt}{0pt plus 2pt minus 2pt}
\titlespacing\subsubsection{0pt}{2pt plus 2pt minus 2pt}{0pt plus 2pt minus 2pt}

\let\oldcenter\center
\let\oldendcenter\endcenter

\renewenvironment{center}
  {\setlength{\topsep}{6pt}%
   \setlength{\partopsep}{6pt}%
   \oldcenter}
  {\oldendcenter}

\setlist[itemize]{
  leftmargin=1.5em,
  itemsep=0pt,
  parsep=0pt,
  topsep=0pt,
  partopsep=0pt
}

\setlist[enumerate]{
  itemsep=0pt,
  parsep=0pt,
  topsep=0pt,
  partopsep=0pt
}

\usepackage{tikz}
\usetikzlibrary{bayesnet}
\usetikzlibrary{arrows}
\usetikzlibrary{arrows.meta,positioning,bending}
\usetikzlibrary{backgrounds}

\theoremstyle{plain}
\newtheorem{theorem}{Theorem}[section]
\newtheorem{proposition}[theorem]{Proposition}
\newtheorem{lemma}[theorem]{Lemma}

\newtheorem{assumption}[theorem]{Assumption}
\newtheorem{condition}[theorem]{Condition}

\theoremstyle{definition}
\newtheorem{definition}[theorem]{Definition}
\theoremstyle{remark}

\usepackage[capitalize,noabbrev]{cleveref}
\crefname{section}{Sec.}{Secs.}
\Crefname{section}{Section}{Sections}
\Crefname{table}{Table}{Tables}
\crefname{table}{Table}{Tables}

\usepackage[most,skins,theorems]{tcolorbox}
\tcbset{
  aibox/.style={
    width=\linewidth,
    top=7pt,
    bottom=2pt,
    colback=rliableolive!13!white,
    colframe=black,
    colbacktitle=black,
    enhanced,
    center,
    attach boxed title to top left={yshift=-0.1in,xshift=0.15in},
    boxed title style={boxrule=0pt,colframe=white,},
  }
}
\newtcolorbox{AIbox}[2][]{aibox,title=#2,#1}

\usepackage{wrapfig}
\definecolor{lightblue}{rgb}{0.22,0.45,0.70}
\usepackage{tabularx}
\definecolor{bestyellow}{RGB}{255,244,204}
\newif\ifnips     
\nipstrue        

\title{\Large World Action Verifier:\\Self-Improving World Models via Forward-Inverse Asymmetry}

\author{
{\Authfont
Yuejiang~Liu$^{\dagger,*}$,
Fan~Feng$^{\ddagger,*}$,
Lingjing~Kong$^{\S,*}$,
Weifeng~Lu\thanks{Equal contribution. Website: \href{https://world-action-verifier.github.io}{\texttt{world-action-verifier.github.io}}} ,
Jinzhou~Tang$^{\ddagger}$,
\vspace{0.2em}
Kun~Zhang$^{\S}$,
Kevin~Murphy$^{\P}$,
Chelsea~Finn$^{\dagger}$,
Yilun~Du$^{\|}$
} \\
\vspace{0.2em}
{\Affilfont
$^{\dagger}$Stanford University \,\,
$^{\ddagger}$UC San Diego \,\,
$^{\S}$Carnegie Mellon University \,\,\\
\vspace{0.2em}
$^{\P}$Google DeepMind \,\,
$^{\|}$Harvard University
}
}
\makeatletter
\providecommand{\@trackname}{}
\makeatother
\date{\today} 

\begin{document}

\addtocontents{toc}{\protect\setcounter{tocdepth}{0}}

\maketitle

\vspace{-16pt}


\begin{abstract}
    General-purpose world models promise scalable policy evaluation, optimization, and planning, yet achieving the required level of robustness remains challenging. Unlike policy learning which primarily focuses on optimal actions, a world model needs to be reliable over a vast space of suboptimal actions, which are often underrepresented in action-labeled robot interactions. To address this challenge, we propose World Action Verifier (WAV), a framework that enables world models to identify their own prediction errors and self-improve. The key idea is to decompose action-conditioned state prediction into two independently verifiable factors: state plausibility and action reachability. We show that verifying these factors is significantly more tractable than direct forward prediction due to two underlying asymmetries: the broader availability of action-free data and the lower dimensionality of action-relevant features. Leveraging these asymmetries, we augment a world model with (i) a diverse subgoal generator obtained from video corpora and (ii) a sparse inverse model that infers actions from a subset of state features. By enforcing cycle consistency among proposed subgoals, inferred actions, and forward rollouts, WAV provides an effective verification mechanism in under-explored regimes, where existing methods often fail. Across nine tasks spanning MiniGrid, RoboMimic, and ManiSkill, our method achieves 2$\times$ higher sample efficiency while improving downstream policy performance by over 22\%. 
\end{abstract}

\section{Introduction}
\label{sec:introduction}

World models—action-conditioned forward dynamics models that predict future states given specific actions or action chunks—have come to play an increasingly important role in robot learning~\citep{ha2018recurrent,wu2022daydreamerworldmodelsphysical,assran2025v,panteam2025panworldmodelgeneral,zheng2025flarerobotlearningimplicit,huang2026pointworldscaling3dworld}. 
Recent works have shown that, when trained on action-labeled robot interactions alongside action-free internet videos~\citep{rigter2024avid,huang2025vid2world,ye2026world}, world models have the potential to not only generate controllable future dynamics but also enable scalable policy evaluation~\citep{quevedo2025worldgymworldmodelenvironment, geminiroboticsteam2026evaluatinggeminiroboticspolicies,zhu2025irasimfinegrainedworldmodel,li2025worldevalworldmodelrealworld}, policy optimization~\citep{hafner2019learning, yang2023learning, yang2024learninginteractiverealworldsimulators, guo2025ctrlworldcontrollablegenerativeworld}, and test-time planning~\citep{hafner2019dream, hafner2023mastering, jain2025smoothseaskilledsailor,zhou2025dinowmworldmodelspretrained,qi2025strengthening}.

Despite remarkable progress, building a general-purpose world model that is robust enough for various downstream applications remains difficult. A central challenge is \emph{action following}: predicting future states that faithfully reflect the effects of the given actions~\citep{shang2026worldarena}.
Unlike policy learning, which primarily focuses on modeling optimal actions, a world model must be reliable across a much broader action distribution, including suboptimal, exploratory, and even random actions encountered during policy learning or evaluation~\citep{lecun2022path, zhang2024whale, jain2025smoothseaskilledsailor}.
However, collecting robot interactions covering diverse actions is often slow, expensive, and sometimes even unsafe.
Given a limited budget of robot data, deciding which specific interactions to collect remains a pressing challenge.

Previous work has sought to address this through two main approaches. One line of work relies on on-policy exploration, \ie, gathering data by rolling out the policies of interest~\citep{jain2025smoothseaskilledsailor,guo2026vlaw,liu2026world}. While effective for the considered policies, the learned model often degrades sharply beyond predefined policy sets, compromising its generality.
Another line of work focuses on info-max exploration, actively seeking interactions that maximize information gain~\citep{pathak2017curiositydrivenexplorationselfsupervisedprediction,sekar2020planning,kim2020activeworldmodellearning}. 
A common proxy for information gain is the prediction error of the world model, estimated before collecting the corresponding transition--a process we refer to as \emph{world model verification}. This verification process, however, often suffers from a practical challenge: existing methods tend to be reliable in well-explored regions where additional data are largely redundant, but unreliable in under-explored regions where verification is critically needed. After all, the interactions that are most informative for exploration are precisely those where the least prior information exists for verification.
This tension raises a central question:
\vspace{-0.6em}
\begin{center}
\emph{How can we reliably verify the predictions of a world model in under-explored regimes?}
\end{center}
\vspace{-0.6em}

To this end, we propose \methodname{} (\methodacro), a framework that enables world models to verify their own predictions and self-improve through an asymmetric forward-inverse cycle. The core idea is to decompose action-conditioned forward predictions into two complementary factors: \emph{state plausibility}, \ie, whether a predicted state is visually realistic, and \emph{action reachability}, \ie, whether the predicted transition is physically achievable under the given actions.
This decomposition not only allows each factor to be verified separately, but also admits two crucial asymmetries: (i) the broader availability of action-free data: state plausibility can be verified using internet videos without action labels, which are far more abundant than the action-labeled robot interactions used to train the world model, and (ii) lower dimensionality of action-relevant features: action reachability can be verified based on a compact subset of state features relevant to the actions, which are much lower-dimensional than the full state the world model must predict.

Motivated by these asymmetries, we augment a world model with two additional components: a diverse subgoal generator obtained from video corpora, and a sparse inverse model trained to infer actions from a learned subset of state features. 
Together, these components induce a goal-oriented self-improvement cycle: the subgoal generator proposes plausible future states, the inverse model infers actions that could reach them, and the forward world model rolls out those actions to test whether the predicted states are consistent with the proposed subgoals~(\cref{fig:teaser}).
Theoretically, we show that verification via a sparse inverse process is easier than dense forward generation, particularly in high-dimensional stochastic environments. Empirically, we evaluate \methodacro{} on nine tasks spanning MiniGrid~\citep{chevalier2023minigrid}, RoboMimic~\citep{zhu2020robosuite}, and ManiSkill~\citep{mu2maniskill}. 
Compared to existing methods, \methodacro{} improves the sample efficiency of world models by $2\times$ and boosts downstream policy performance by more than 22\%. Our results suggest that the asymmetries between forward and inverse dynamics offer a promising ingredient for building self-improving world models.

\begin{figure}[t]
    \vspace{-8pt}   
    \centering
    \small
    \includegraphics[width=0.95\linewidth]{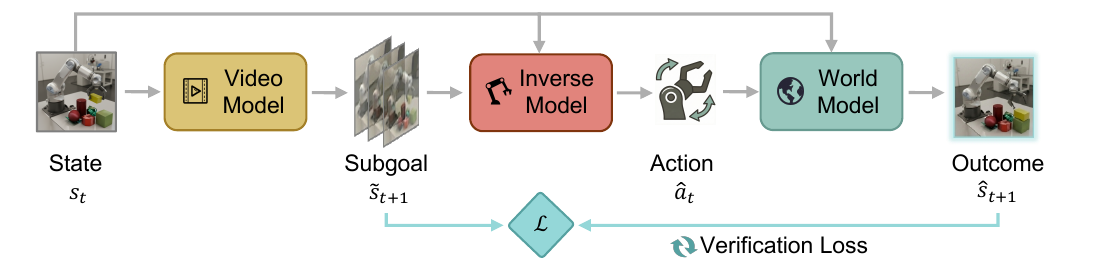}
    \caption{Overview of \methodname{}, a framework that enables action-conditioned world models to verify their predictions and self-improve from an asymmetric forward-inverse cycle: (i) a {\em diverse} subgoal generator proposes plausible future states, (ii) a {\em sparse} inverse model infers actions from a relevant subset of state features, and (iii) a world model rolls forward and verifies consistency between its predicted state and the proposed state.}
    \label{fig:teaser}
\end{figure}

\section{Method: \methodname{} for Self-Improving World Models}
\label{sec:method}

World models excel when grounded in action-labeled interaction data, yet collecting such data at scale is often prohibitively expensive. 
In this section, we present World Action Verifier (WAV), a self-improving framework that enables a world model to verify its own predictions and prioritize informative exploration. 
We first formalize the verification problem in a semi-supervised setting~(\cref{subsec:setting}), then decompose it into two more tractable subproblems~(\cref{subsec:verification}), and finally couple them into a goal-oriented exploration procedure for self-improvement~(\cref{subsec:selfimprove}).

\subsection{Preliminary: Semi-Supervised Verification of World Models}
\label{subsec:setting}

We consider a world model $f_\theta$ as an action-conditioned forward dynamics model,
$
\hat{s}^{t+1} = f_\theta(s^t, a^t),
$
where $s^t$ and $a^t$ are the state and action, or action chunk, at time $t$, and $\hat{s}^{t+1}$ is the predicted successor state.
Following recent training recipes~\citep{huang2025vid2world,gao2026dreamdojo}, we study a semi-supervised setting with two data sources: a small action-labeled robot interaction dataset
$
\mathcal{D}_{\mathrm{act}}=\{(s^t,a^t,s^{t+1})\}
$
and a large action-free video dataset
$
\mathcal{D}_{\mathrm{vid}}=\{(s^t,s^{t+1},\ldots)\}.
$
Typically, $\mathcal{D}_{\mathrm{vid}}$ spans a much broader range of state transitions than $\mathcal{D}_{\mathrm{act}}$.

Our goal is to improve $f_\theta$ not only on the narrow action distribution represented in $\mathcal{D}_{\mathrm{act}}$, but also on the broader transition support reflected in $\mathcal{D}_{\mathrm{vid}}$.
However, the lack of action labels in online videos poses a key challenge for \emph{action following}: rather than faithfully grounding predictions in the conditioning action, existing world models often hallucinate future states that may look visually plausible but are physically misaligned with the given action~\citep{shang2026worldarena,mei2026video}.
A natural remedy is to collect additional action-labeled robot interactions~\citep{jain2025smoothseaskilledsailor,guo2026vlaw,liu2026world}. 
Yet, since large-scale robot interaction data are costly to collect, a critical question arises: \emph{which specific interactions should be prioritized to improve the world model most effectively?}

Intuitively, transitions that the model can already predict accurately yield little new knowledge.
Instead, the data budget should be steered toward transitions that are likely to induce large prediction errors.
More formally, for a transition $(s^t,a^t,s^{t+1})$, we define the true
prediction error as
\begin{equation}
    \varepsilon(s^t,a^t; s^{t+1})
    :=
    \ell\!\left(f_\theta(s^t,a^t), s^{t+1}\right),
    \label{eq:pred_error}
\end{equation}
where $\ell(\cdot,\cdot)$ is a discrepancy measure in the state space.
Since the true successor state $s^{t+1}$ cannot be observed prior to execution, we aim to construct a \emph{verifier} $\hat{\varepsilon}(s^t, a^t, \hat{s}^{t+1})$ that estimates this error.
From an exploration standpoint, the verifier need not be perfectly calibrated, but it should preserve the relative ranking of prediction errors across candidate actions~\citep{houthooft2017vimevariationalinformationmaximizing,saravanan2026diffbed}.
That is, given two candidate actions $a_i^t$ and $a_j^t$, with predicted successor states $\hat{s}_i^{t+1}=f_\theta(s^t,a_i^t)$ and $\hat{s}_j^{t+1}=f_\theta(s^t,a_j^t)$, an effective verifier for exploration should satisfy
\begin{equation}
    \varepsilon(s^t,a_i^t; s_i^{t+1})
    <
    \varepsilon(s^t,a_j^t; s_j^{t+1})
    \;\Longrightarrow\; 
    \hat{\varepsilon}(s^t,a_i^t,\hat{s}_i^{t+1}) 
    < 
    \hat{\varepsilon}(s^t,a_j^t,\hat{s}_j^{t+1}).
    \label{eq:ranking_goal}
\end{equation}

\subsection{Two Complementary Factors of Verification}
\label{subsec:verification}

A common approach to verifying forward predictions in~\cref{eq:ranking_goal} is to leverage the internal knowledge of the world model itself, \eg, extracting epistemic uncertainty from a single model~\citep{pathak2017curiositydrivenexplorationselfsupervisedprediction} or measuring disagreement across multiple models~\citep{sekar2020planning,kim2020activeworldmodellearning}.
However, such verification methods often inherit the blind spots of the learned world model: they provide relatively reliable error estimates in well-explored regimes where the current world model is already accurate, but become much less reliable in under-explored regimes where accurate verification is most critical.

To overcome this issue, we take a different perspective: rather than directly verifying the overall correctness of a forward prediction, we decompose it into sub-conditions that are easier to verify.
More specifically, motivated by the Bayes decomposition,
\begin{equation}
    p(s^{t+1}\mid s^t,a^t)
    =
    \frac{p(a^t\mid s^t,s^{t+1})\,p(s^{t+1}\mid s^t)}{p(a^t\mid s^t)}
    \;\propto\;
    \underbrace{p(s^{t+1}\mid s^t)}_{\text{state}}\,
    \underbrace{p(a^t\mid s^t,s^{t+1})}_{\text{action}},
    \label{eq:bayes_decomp}
\end{equation}
we view a correct action-conditioned forward prediction as satisfying two complementary criteria:
\begin{itemize}[nosep]
    \item \textit{State Plausibility}: whether the predicted next state is plausible under the environment dynamics.
    \item \textit{Action Reachability}: whether the transition from $s^t$ to $s^{t+1}$ is consistent with the given action.
\end{itemize}
From this view, a correct prediction should both lie on the manifold of plausible futures and be reachable under the specified action.
Crucially, each condition admits a verification strategy that is more tractable than predicting high-dimensional forward dynamics, which we will describe next. 

{\bf State verification via distribution asymmetry.}
One common failure mode of high-dimensional forward prediction is poor visual plausibility.
For example, a world model post-trained on limited robot interaction data may partially forget the general dynamics learned from video pretraining, resulting in blurry or inconsistent rollouts~\citep{zhang2026physion}.
To detect such errors, we build a state verifier from the first factor in~\cref{eq:bayes_decomp}, namely a state transition prior $g_\phi(s^{t+1}\mid s^t)$.
Since this prior does not condition on actions, it can be trained not only on action-labeled robot interactions $\mathcal{D}_{\mathrm{act}}$, but also on the much larger action-free video dataset $\mathcal{D}_{\mathrm{vid}}$.
Moreover, as the prediction error of a world model often compounds over longer horizons, we instantiate $s^{t+1}$ as a future subgoal reached after an action chunk rather than as the immediate next frame.
After training, $g_\phi$ allows us to sample a diverse set of $K$ plausible future subgoals for state verification,
\begin{equation}
    \{\tilde{s}^{t+1}_k\}_{k=1}^K \sim g_\phi(\cdot \mid s^t).
    \label{eq:subgoal_generation}
\end{equation}

{\bf Action verification via dimensionality asymmetry.}
State plausibility alone does not imply correct forward prediction: for downstream policy use, predicted transitions must also be reachable under the specified action.
To assess this condition, we build another verifier from the second factor in~\cref{eq:bayes_decomp}, namely an inverse dynamics model $h_\psi(a^t\mid s^t,s^{t+1})$ that infers which action could connect a current state to a future state.
While the inverse model $h_\psi$ is action-dependent and cannot leverage more training data than the forward world model, it benefits from a lower effective dimensionality in both its outputs and its relevant inputs.
For instance, in visually complex scenes, an action typically affects only one object at a time rather than the entire scene.
Inspired by the observation that models attending to a compact set of causally relevant features often generalize better~\citep{liu2023causal,wang2025dyn}, we explicitly impose a learnable sparsity mask $M$ in the inverse dynamics model:
\begin{equation}
    \hat{a}^{t} = h_\psi\!\left( M \odot s^{t},\; M \odot s^{t+1} \right).
    \label{eq:sparse_inverse_def}
\end{equation}
As illustrated in~\cref{fig:decomposition}, the subgoal generator $g_\phi$ and inverse model $h_\psi$ provide two complementary components for verification: the former checks whether a candidate future is plausible, while the latter checks whether it is reachable through an inferred action.
\ifnips
    \begin{figure}[t]
    \centering
    \begin{minipage}[t]{0.45\linewidth}
        \vspace{-6pt} 
        \centering
        \includegraphics[width=0.95\linewidth]{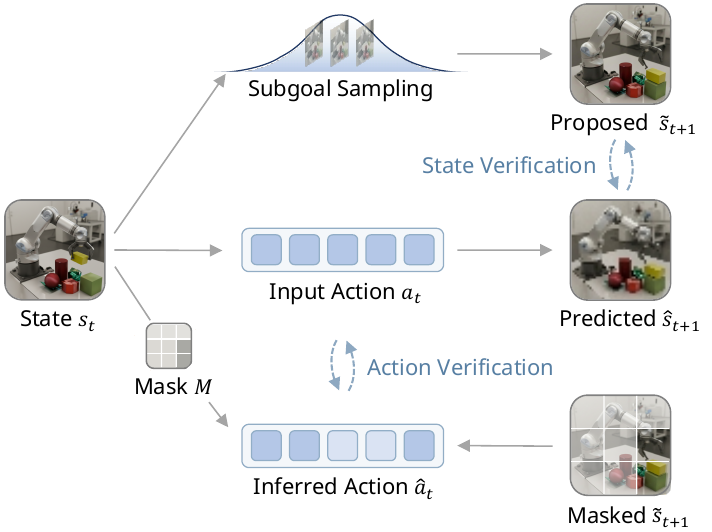}
        \vspace{-6pt}
        \caption{Decomposing model verification into state plausibility and action reachability.}
        \label{fig:decomposition}
    \end{minipage}
    \hfill
    \begin{minipage}[t]{0.525\linewidth}
        \vspace{0pt} 
            
        \vspace{-4pt}
        \rule{\linewidth}{0.8pt}
        \captionsetup{
          justification=raggedright,
          singlelinecheck=false
        }
        \captionof{algorithm}{\small WAV-Guided Exploration.}
        \vspace{-4pt}
        \label{alg:exploration}
        \rule{\linewidth}{0.4pt}
    
\begin{lstlisting}[basicstyle=\fontsize{7.75pt}{8pt}\ttfamily, xleftmargin=-2pt]
# f: world model, h: inverse model
# s: current state, g: subgoal generator
# D: current data, K: number of candidates

for each exploration iteration:
  s_g = v.sample(s, K)     # subgoals
  a = h.inverse(s, s_g)    # actions
  s_p = f.predict(s, a)    # outcomes
  scores = dist(s_g, s_p)  # disagreement
  idx = argmax(scores)     # max surprise
  s_n = env.step(a[idx])
  D.append((s, a[idx], s_n)) 
  f.update(D), h.update(D)   
\end{lstlisting}
    
    \vspace{-6pt}
    \rule{\linewidth}{0.8pt}
    \end{minipage}
    
    \end{figure}
    
\else
    \begin{figure}[t]
    \centering
    \begin{minipage}[t]{0.475\linewidth}
        \vspace{-6pt} 
        \centering
        \includegraphics[width=0.90\linewidth]{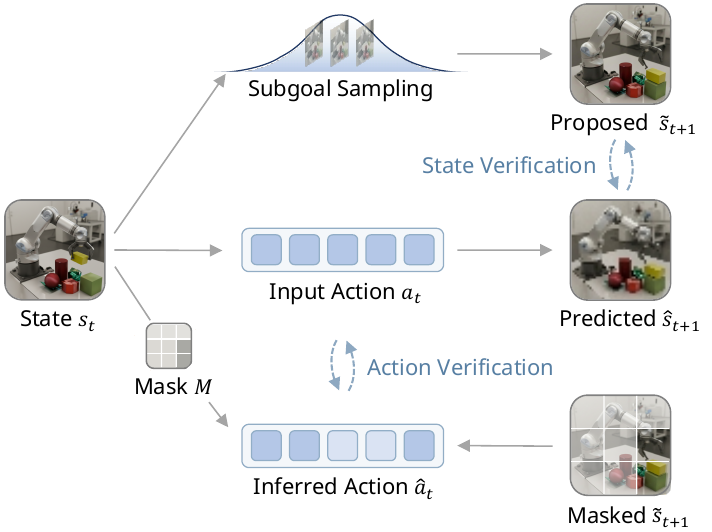}
        \vspace{-6pt}
        \caption{Decomposing world model verification into state plausibility and action reachability.}
        \label{fig:decomposition}
    \end{minipage}
    \hfill
    \begin{minipage}[t]{0.50\linewidth}
        \vspace{0pt} 
            
        \vspace{-4pt}
        \rule{\linewidth}{0.8pt}
        \captionof{algorithm}{\small WAV-Guided Exploration.}
        \vspace{-4pt}
        \label{alg:exploration}
        \rule{\linewidth}{0.4pt}
    
        \begin{lstlisting}
# s: current state, f: world model
# v: subgoal generator, h: inverse model
# D: current data, K: number of candidates

for each exploration iteration:
  s_g = v.sample(s, K)       # subgoals
  a = h.inverse(s, s_g)      # actions
  s_p = f.predict(s, a)      # outcomes
  scores = dist(s_g, s_p)    # disagreement
  idx = argmax(scores)       # max surprise
  s_n = env.step(a[idx])     # collect data
  D.append((s, a[idx], s_n)) 
  f.update(D), h.update(D)   
\end{lstlisting}
    
    \vspace{-6pt}
    \rule{\linewidth}{0.8pt}
    \end{minipage}
    
    \end{figure}

\fi

\subsection{Goal-Oriented Exploration}
\label{subsec:selfimprove}

Given the two verification criteria above, we next couple them into a verification-driven exploration algorithm.
A straightforward design is action-oriented exploration: sample candidate actions, roll out the forward world model $f_\theta$, and then use the inverse model $h_\psi$ to check whether the generated state transitions recover the original actions~\citep{ye2025reinforcement}.
However, this ordering can be brittle in practice.
Among the three components, the forward world model $f_\theta$ is often the least reliable when action-labeled data are scarce.
As such, errors introduced by the initial forward rollout can produce off-manifold states, on which the subsequent inverse model also becomes unreliable.

We therefore pursue a goal-oriented alternative that places the forward world model as the final step in the verification cycle.
At each time step $t$, we first sample plausible subgoals from the transition prior, infer actions that could reach those subgoals, and only then verify whether the action-conditioned world model can realize them:
\begin{equation}
    s^t
    \fixedxrightarrow{\,g_\phi\,}
    \tilde{s}_{1:K}^{t+1}
    \fixedxrightarrow{\,h_\psi\,}
    \hat{a}_{1:K}^{t}
    \fixedxrightarrow{\,f_\theta\,}
    \hat{s}_{1:K}^{t+1}
    \fixedxrightarrow{\,\ell\,}
    \hat{\varepsilon}_{1:K}
    \fixedxrightarrow{\,\max\,}
    a^{\star}
    .
\end{equation}
For each candidate subgoal $\tilde{s}_{k}^{t+1}$, we measure how far the forward rollout $\hat{s}_k^{t+1}$ deviates from it.
In practice, this discrepancy $\ell$ can be computed in the discrete state space, a continuous representation space, or a diffusion noise space, depending on the parameterization of the world model.

As summarized in~\cref{alg:exploration}, we execute $a^\star$ associated with the largest discrepancy at each time step, add the resulting transition to $\mathcal{D}_{\mathrm{act}}$, and iteratively update both the forward world model and the inverse dynamics model.
By verifying multiple candidate rollouts in parallel before acting, our method reduces unnecessary real-world interaction and thereby effectively trades scalable computation for improved data efficiency.

\section{Theory: When Does Inverse Verification Outperform Forward Prediction?}
\label{subsec:sample_complexity}

Our method in~\cref{sec:method} rests on a central premise: verifying actions through inverse dynamics can be substantially easier than predicting full future states with a forward world model. 
In this section, we provide a minimal theoretical analysis characterizing the conditions under which this forward--inverse asymmetry becomes most pronounced.

To isolate the key factors, we consider a linear--Gaussian setting with the observed state $s \in \R^{d_s}$ and action $a \in \R^{d_a}$.
Assume one-step dynamics
\begin{equation}
    s' = As + Ba + \xi,
    \qquad
    \xi \sim \mathcal{N}(0,\sigma_s^2 \Id_{d_s}),
    \label{eq:analysis_fwd_main}
\end{equation}
where $\sigma_s$ captures transition stochasticity. We further assume that the action is recoverable from a low-dimensional action-relevant slice $z := Ms \in \R^{d_z}$ with $d_z \lll d_s$:
\begin{equation}
    a = h(z,z') + \eta,
    \qquad
    h(z,z') := H
    \begin{bmatrix}
    z \\
    z'
    \end{bmatrix},
    \qquad
    \eta \sim \mathcal{N}(0,\sigma_a^2 \Id_{d_a}),
    \label{eq:analysis_inv_main}
\end{equation}
where $z' := Ms'$ and $\sigma_a$ measures irreducible ambiguity in recovering the action from $(z,z')$.

We compare a dense forward model $f_\theta$ trained on $[s;a] \in \R^{d_s+d_a}$ against a sparse inverse model $h_\psi$ trained on $[z;z'] \in \R^{2d_z}$, both fit by ordinary least squares (OLS) on $n$ transitions from $\mathcal{D}_{\mathrm{act}}$. To compare them in the same units, we evaluate in state space:
\begin{equation}
    \mathcal{E}_F := \frac{1}{d_s}\E\!\left[\|f_\theta(s,a)-f(s,a)\|_2^2\right],
    \quad
    \mathcal{E}_I := \frac{1}{d_s}\E\!\left[\|f(s,h_\psi(z,z'))-f(s,h(z,z'))\|_2^2\right],
    \label{eq:analysis_metrics_main}
\end{equation}
where $f(s,a)$ denotes the true dynamics.  Since inverse errors enter the state prediction through the action channel, we define
$\lambda := \|B\|_{\op}$ as the worst-case amplification from action error to state error.

\begin{proposition}[Informal]
\label{prop:forward_inverse_gap_informal}
Under the stylized setup above, if both models are fit by OLS on $n$ labeled transitions, then
\begin{equation}
\frac{\E[\mathcal{E}_F]}{\E[\mathcal{E}_I]}
\;\ge\;
\underbrace{\Bigl(\frac{d_s+d_a}{2d_z}\cdot\frac{d_s}{d_a}\Bigr)}_{\text{dimensionality}}
\cdot
\underbrace{\Bigl(\frac{\sigma_s}{\lambda\,\sigma_a}\Bigr)^2}_{\text{stochasticity}}
\cdot
\underbrace{\Bigl(\frac{n-2d_z-1}{n-(d_s+d_a)-1}\Bigr)}_{\text{sample size}},
\label{eq:analysis_gap_main}
\end{equation}
provided $n > d_s+d_a+1$ and $n > 2d_z+1$. The exact statement and proof are in Appendix~\ref{app:sample_complexity}.
\end{proposition}

{\bf Interpretation.}
The ratio in \eqref{eq:analysis_gap_main} factorizes into three terms.
(1) \emph{Dimensionality}: the forward model must estimate a map from $d_s+d_a$ inputs, whereas the sparse inverse uses only $2d_z$.
(2) \emph{Stochasticity}: forward prediction suffers from environment noise $\sigma_s$, while inverse verification suffers only from action-recovery ambiguity $\sigma_a$ (scaled by $\lambda$).
(3) \emph{Sample size}: when $n$ is only modestly larger than $d_s+d_a$, the forward estimator is far less stable.
In practice, \methodacro helps most when
(i) the verifier needs only a small agent-centric subset while the world model predicts a large scene (\emph{large $d_s/d_z$});
(ii) uncontrolled dynamics inflate $\sigma_s$ while the action imprint stays clean (\emph{large $\sigma_s/\sigma_a$}); and
(iii) action-labeled data are limited (\emph{small $n$}).
We validate each factor empirically in \cref{sec:robustness_of_self_verification}: varying the data budget isolates the sample-size term, increasing the number of objects raises the effective state dimension $d_s$, and adding noisy floors inflates $\sigma_s$ while leaving $\sigma_a$ unchanged.

\section{Experiments}
\label{sec:experiments}
In this section, we evaluate the central claims of WAV. We begin by validating whether inverse verification, especially with sparsity, is more robust and easier to learn than action-conditioned forward prediction under limited data and distribution shift. We then examine whether WAV can effectively improve world model learning through self-improvement, and finally, whether the resulting gains translate into better downstream policy learning and enhance out-of-distribution generalization. 
Concretely, we study the following research questions:
\begin{itemize}
    \item {\textit{RQ1}}: Is learning an inverse dynamics model easier than learning a forward world model?
    \item {\textit{RQ2}}: Do sparse IDMs generalize better than vanilla IDMs to unseen objects or interactions?
    \item {\textit{RQ3}}: How effective is forward--inverse asymmetry for self-improving the world model? 
    \item {\textit{RQ4}}: Does this self-improvement translate into improved downstream policy learning?
    \item {\textit{RQ5}}: Can the learned world model adapt to out-of-distribution robotic manipulation settings with limited target-domain data under novel visual setups, objects, and interactions?

\end{itemize}
    
\ifnips
\else

To address RQ1 and RQ2, we conduct controlled experiments in MiniGrid~\citep{chevalier2023minigrid}, with a particular emphasis on out-of-distribution generalization to unseen objects and novel interactions. 
Building on these findings, we then evaluate the proposed framework in more complex settings, assessing how it enables self-improvement of world models quality (\textit{RQ3}) on both MiniGrid and simulated robotic manipulation tasks (RoboMimic~\citep{zhu2020robosuite} and ManiSkill~\citep{mu2maniskill}), and how these enhance downstream policy learning performance (\textit{RQ4}) on robotic manipulation tasks. Finally, we further evaluate whether the learned world model can adapt to out-of-distribution robotic settings with limited target-domain data (\textit{RQ5}), considering both visual appearance shifts and more challenging interaction distributions.
\fi

\textbf{Baselines.}
We compare our method (\cref{sec:method}) against the following exploration strategies:
\begin{itemize}
    \item \textit{Random:} randomly samples candidate interactions, serving as a lower-bound exploration strategy.
    \item \textit{Uncertainty:} select candidates with the highest predictive uncertainty~\citep{sekar2020planning}.
    \item \textit{Progress:} selects candidates with the largest learning progress, measured by the disagreement between two consecutive world models during training~\citep{kim2020activeworldmodellearning}.
    \item \textit{Vanilla IDM}: our method without the sparsity mask $M$ in~\cref{eq:sparse_inverse_def}.
    \item \textit{Oracle:} selects candidates with the largest world-model prediction loss computed using ground-truth successor states, serving as an upper bound.
\end{itemize}

\ifnips
    \begin{figure}[t]
        \centering
        \makebox[\linewidth][c]{%
            \includegraphics[width=0.45\linewidth]{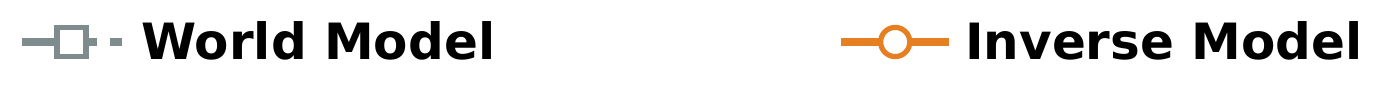}%
        }
        \vspace{-4mm}
        
        \includegraphics[height=3.3cm]{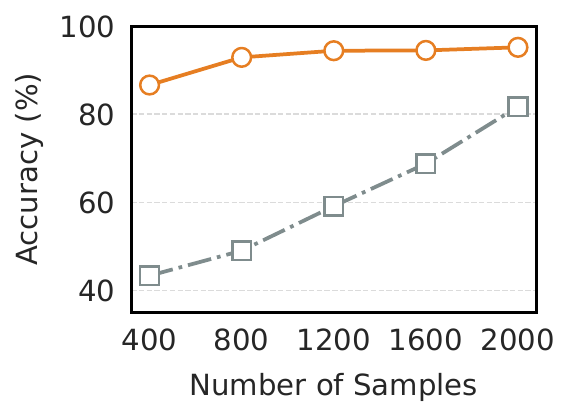}\hfill
        \includegraphics[height=3.3cm]{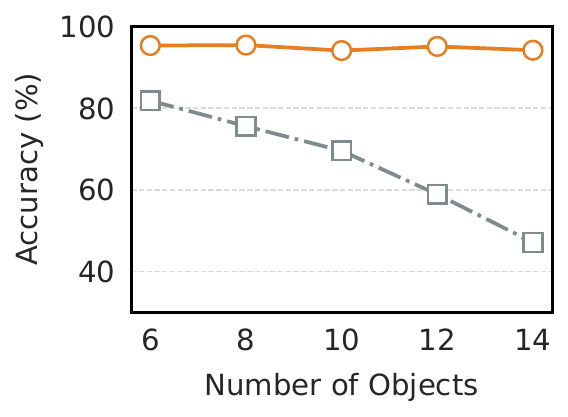}\hfill
        \includegraphics[height=3.3cm]{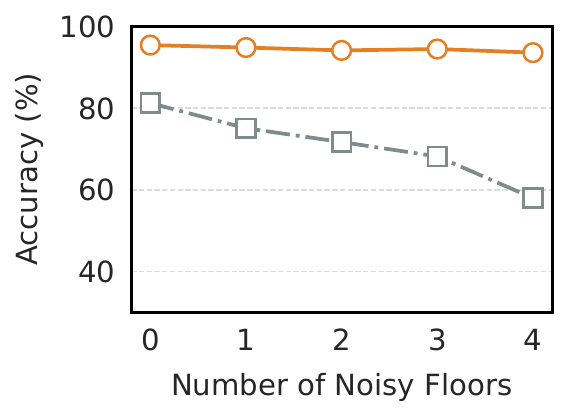}
        
        \caption{
        \textbf{Robustness verification of WAV on MiniGrid.}
        {(Left)} Sample efficiency comparison between Sparse IDM and the World Model with six objects.
        {(Mid)} Robustness to increasing state complexity.
        {(Right)} Robustness to growing environment stochasticity.
        }
        \label{fig:minigrid_results_1}
    \end{figure}
\else
    \begin{figure}[t]
        \centering
        \makebox[\linewidth][c]{%
            \includegraphics[width=0.45\linewidth]{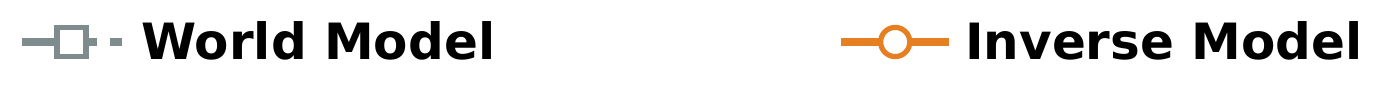}%
        }
        \vspace{-4mm}
        
        \includegraphics[height=4.0cm]{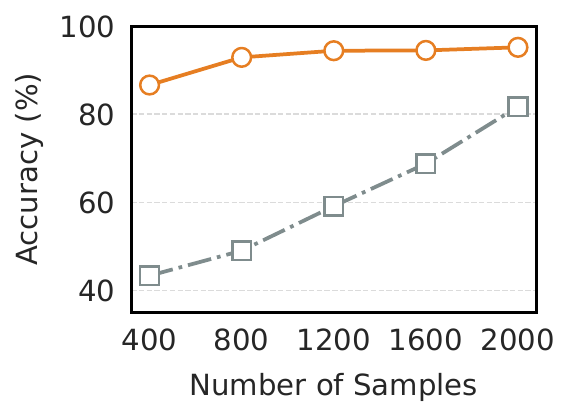}\hfill
        \includegraphics[height=4.0cm]{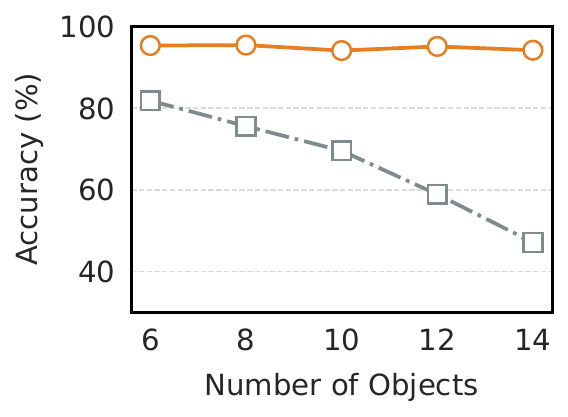}\hfill
        \includegraphics[height=4.0cm]{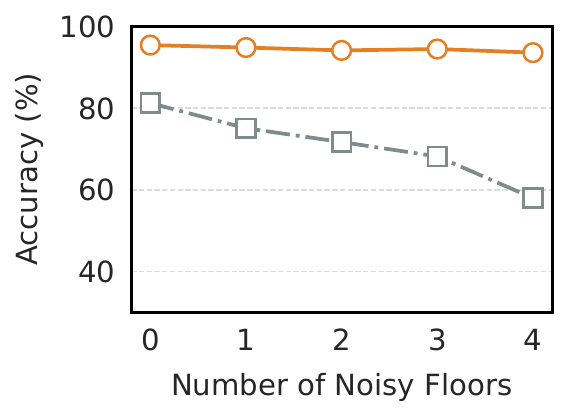}
        
        \caption{
        \textbf{Verification of robustness of WAV on MiniGrid.}
        {(Left)} Sample efficiency comparison between Sparse IDM and the World Model with six objects.
        {(Mid)} Robustness to increasing state complexity.
        {(Right)} Robustness to growing environment stochasticity.
        }
        \label{fig:minigrid_results_1}
    \end{figure}
\fi

\subsection{Experiments in Synthetic MiniGrid}

\textbf{Dataset.}
We collect 50k interaction sequences from three custom MiniGrid tasks: \texttt{Key Delivery}, \texttt{Ball Delivery}, and \texttt{Object Matching}, using a deterministic policy for each task.
Half of the sequences are used to train the action-free subgoal generator as a video prior.
The remaining data form an exploration pool, consisting of an action-labeled seed set of 200 sequences and an unlabeled candidate set of 20k sequences for acquisition.
To enable controlled evaluation, we construct additional random-play datasets by varying object counts and environmental stochasticity.
Specifically, we vary the number of objects to study generalization under increasing scene complexity, and introduce noisy floor tiles whose colors change after each action to simulate stochastic observations.
Data samples from these noisy environments are used exclusively for robustness evaluation.
Additional details are provided in Appendix~\ref{sec:minigrid_setting}.

\subsubsection{Robustness of World Action Verification }
\label{sec:robustness_of_self_verification}
To test whether WAV is a reliable verifier, we compare forward world models and inverse dynamics models under controlled distribution shifts, addressing \textit{\textbf{RQ1}} and \textit{\textbf{RQ2}}.

\textbf{Setup.}
We vary the amount of labeled training data $\{400,800,1200,1600,2000\}$ collected in environments with $6$ objects, and evaluate on test transitions in environments with $\{6,8,10,12,14\}$ objects.
In addition, to examine robustness against observation noise, we construct training datasets in $6$-object environments with $\{0,1,2,3,4\}$ noisy floors.
For direct comparison, we convert inverse model predictions into next state predictions: for each test pair, the IDM predicts an action, which we then execute in the simulator to obtain the induced next state. We report the same dynamics accuracy defined in Appendix~\ref{sec:minigrid_metrics} for both the world model and the IDM-induced transition.


\textbf{Results.}
For \textbf{\textit{RQ1}}, Figure~\ref{fig:minigrid_results_1} empirically validates the three factors identified in Proposition~\ref{prop:forward_inverse_gap_informal}, each isolated by a separate controlled variable.
Figure~\ref{fig:minigrid_results_1} (Left) isolates the \emph{sample-size} factor: across data regimes, action inference using IDMs consistently outperforms learning action-conditioned world models, with the performance gap being most pronounced in the low-data regime, consistent with the diverging finite-sample term in \eqref{eq:analysis_gap_main} when $n$ is only modestly larger than $d_s+d_a$.
Figure~\ref{fig:minigrid_results_1} (Mid) isolates the \emph{dimensionality} factor: increasing the number of objects raises the effective state dimension $d_s$ while leaving the action-relevant subset $d_z$ unchanged. The WM's performance degrades rapidly as state complexity increases, whereas the IDM remains stable, consistent with the growing ratio $(d_s+d_a)/2d_z$ in the dimensionality term.
Figure~\ref{fig:minigrid_results_1} (Right) isolates the \emph{stochasticity} factor: noisy floor tiles inflate observation noise $\sigma_s$ while the action imprint on the agent-centric features remains clean ($\sigma_a$ unchanged). The WM exhibits clear sensitivity to the induced observation noise, whereas the IDM maintains largely invariant performance, consistent with the $(\sigma_s/\lambda\sigma_a)^2$ ratio in the stochasticity term.
These results confirm the predicted advantage of sparse inverse verification across all three factors, providing empirical justification for \textit{using the IDM as a reliable verifier of the world model}.

For \textbf{\textit{RQ2}}, Figure~\ref{fig:minigrid_results_2} (Left) reveals a pronounced divergence in out-of-distribution generalization when data is limited. The vanilla IDM struggles on \texttt{toggle} and \texttt{swap}, whereas the sparse IDM maintains strong performance, showing that enforcing sparsity promotes more robust action inference.

\begin{figure}[t]
    \centering

    \includegraphics[width=0.3\linewidth]{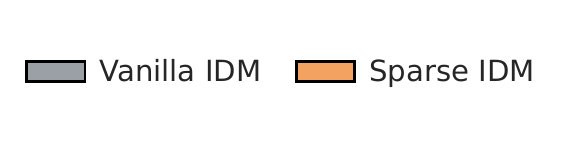}%
    \quad
    \makebox[0.6\linewidth][l]{\hspace{0.04\linewidth}\includegraphics[width=0.6\linewidth]{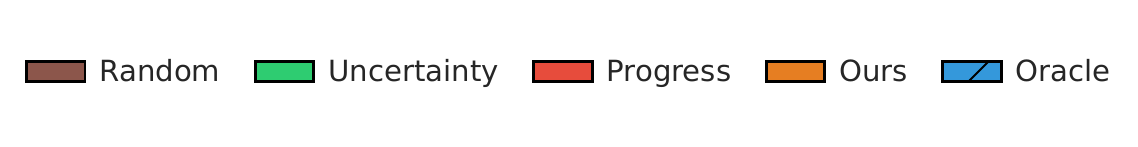}}
    \vspace{-14pt}

    \begin{minipage}[t]{0.32\linewidth}
        \centering
        \includegraphics[width=\linewidth]{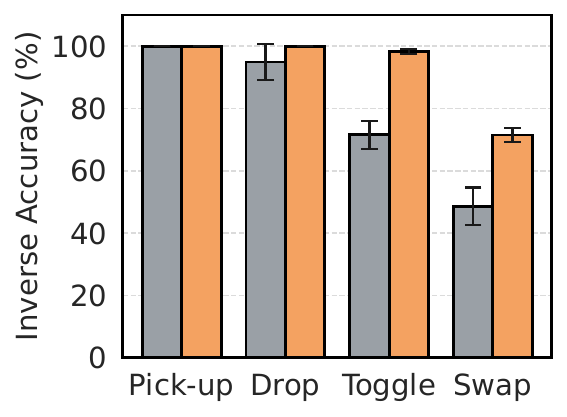}
    \end{minipage}\hfill
    \begin{minipage}[t]{0.32\linewidth}
        \centering
        \includegraphics[width=\linewidth]{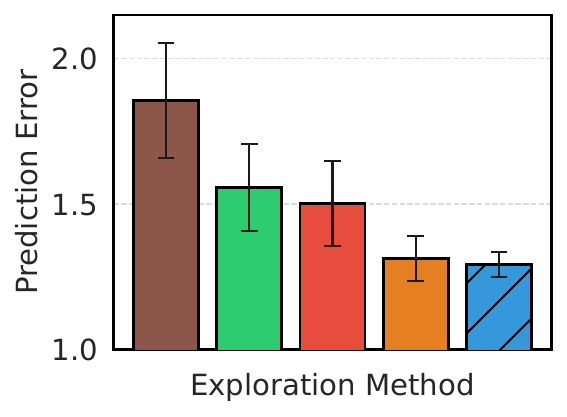}
    \end{minipage}\hfill
    \begin{minipage}[t]{0.32\linewidth}
        \centering
        \includegraphics[width=\linewidth]{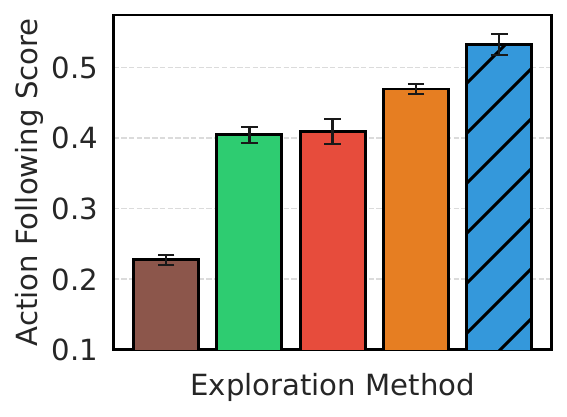}
    \end{minipage}
    
    \caption{
    \textbf{Evaluation of world model learning with WAV on MiniGrid.}
    {(Left)} Prediction accuracy of Sparse IDM and Vanilla IDM. 
    {(Mid)} Comparison of different exploration methods. 
    {(Right)} Robustness of action following across different methods. 
    Results are averaged across 5 seeds.
    }
    \label{fig:minigrid_results_2}
\end{figure}

\subsubsection{Effectiveness of World Model Learning}
\label{sec:efficiency_of_self_exploration}
Building on the above justifications, we now evaluate \textbf{\textit{RQ3}} by examining whether the proposed framework improves world model learning quality.

\textbf{Setup.}
We first train a base model on 200 uniformly sampled labeled transitions, followed by three exploration rounds where each strategy acquires a budget of 100 transitions. We report the prediction error averaged over five random seeds, focusing on the second round where differences are most pronounced for clearer comparison.
To further assess whether the learned models capture action-dependent dynamics, we additionally evaluate an \textit{Action Following Score} (defined in Appendix~\ref{sec:action_following}), using the models obtained from the same active learning round.

\textbf{Results.}
Figure~\ref{fig:minigrid_results_2} (Mid) shows that our method and the Oracle achieve the best performance in terms of prediction error. Consistently, Figure~\ref{fig:minigrid_results_2} (Right) demonstrates that our method also attains the highest Action Following Score, indicating superior modeling of action-dependent dynamics.
We attribute this advantage to the structural imbalance of the dataset, where complex interaction actions are sparse relative to simple movement. The \textit{Random} strategy fails to sample these critical sparse events with sufficient frequency. The \textit{Progress} method struggles with sample redundancy; it tends to over-prioritize transitions where the model is already competent, yielding negligible marginal information gain. Similarly, the \textit{Uncertainty} baseline suffers from a slow warm-up, leading to suboptimal performance in the early stages of exploration.
In contrast, our method prioritizes transitions with high disagreement between the video prior and world model predictions, naturally favoring sparse interaction actions under the same data budget. Moreover, the improved Action Following Score suggests that the learned model better preserves distinctions between different actions, rather than collapsing them into similar predictions. Qualitative results are given in Appendix~\ref{sec:Qualitative Results of MiniGrid}.

\ifnips
\else
\begin{figure}[t]
    \centering
\includegraphics[width=.92\linewidth]{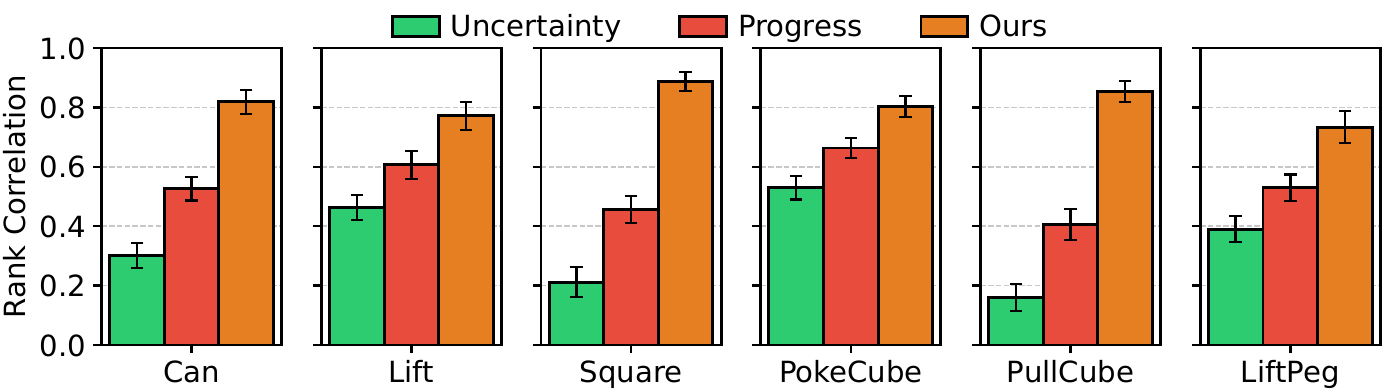}
    \caption{\textbf{Robustness verification of WAV on Robomimic and ManiSkill. }Correlation with Oracle ranking. We evaluate how well each method orders informative samples by computing Spearman rank correlations between the method’s assigned scores and Oracle scores on RoboMimic and ManiSkill environments. Higher correlation indicates closer agreement with the Oracle’s ranking.}
    \label{fig:corr_robo}
\end{figure}

\fi

\subsection{Experiments on Simulated Robot Manipulations}


\textbf{Datasets \& Setups.} We consider a set of challenging robotic manipulation tasks from two evaluation suites: RoboMimic~\citep{zhu2020robosuite} (\texttt{Lift}, \texttt{Can}, \texttt{Square}) and ManiSkill~\citep{mu2maniskill} (\texttt{PullCube}, \texttt{PokeCube}, \texttt{LiftPeg}).
For both suites, we curate training data using expert demonstrations in a two-stage process. We first pretrain diffusion policies~\citep{chi2025diffusion} for different numbers of training steps, yielding a diverse collection of behavior trajectories with varying levels of optimality. Based on these trajectories, we partition the data into \textit{two subsets}: (1) \textit{the warm-up dataset}, which includes the expert demonstrations together with on-policy trajectories collected from the best-performing diffusion policy checkpoint trained on those demonstrations; (2) \textit{the exploration dataset}, which consists of trajectories generated by imperfect diffusion policy checkpoints, capturing diverse exploratory behaviors.

\begin{figure*}[t]
    \centering
    \includegraphics[width=1.0\linewidth]{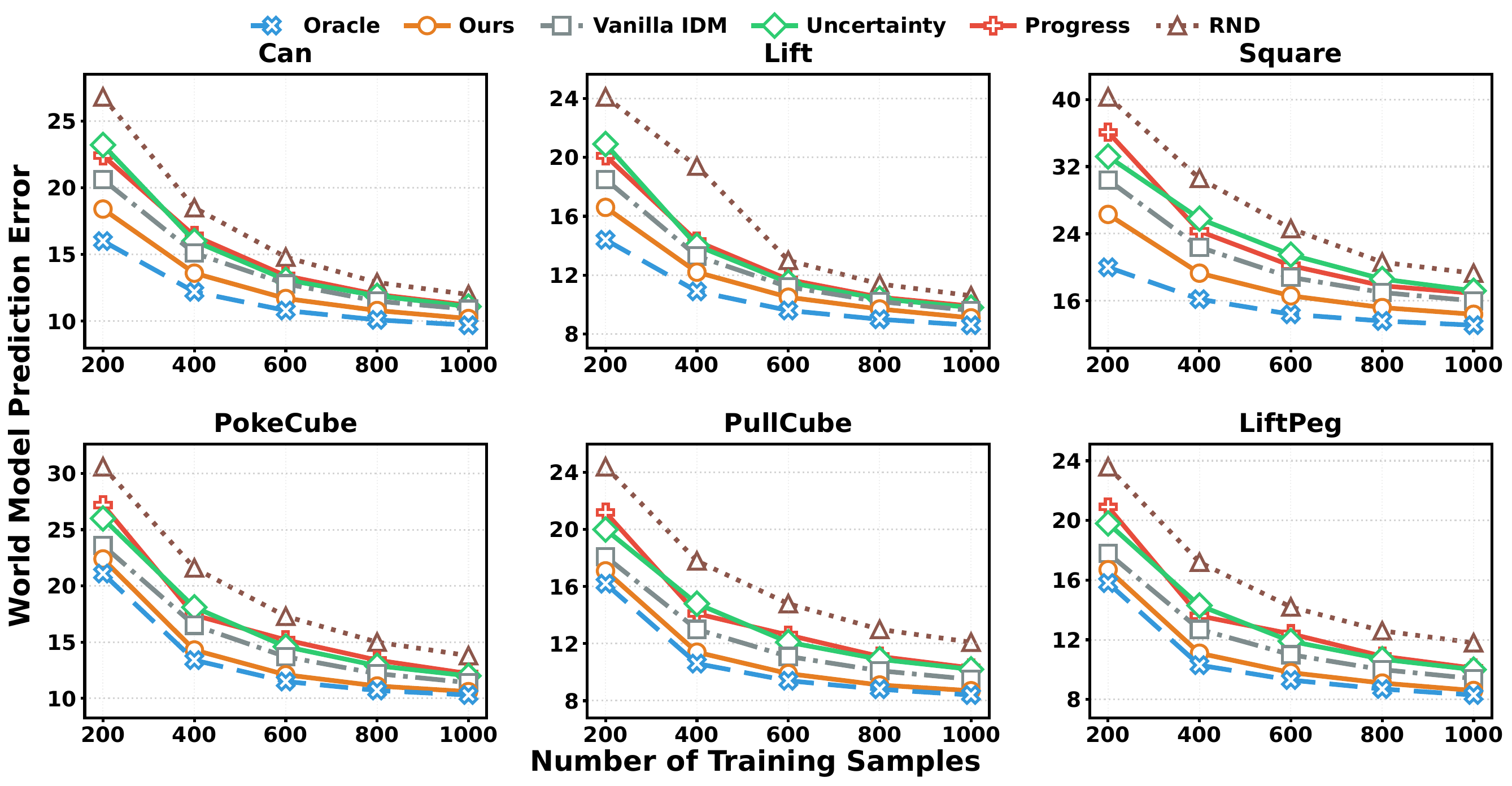}
    \caption{\textbf{World-model learning on RoboMimic and ManiSkill.}
    We report 32-frame prediction error (MSE) as the number of training trajectories increases over 3 seeds.}
    \label{fig:wm_pred_robots}
\end{figure*}


\textbf{Model Choices.} For the world model, we adopt Dreamer-v3~\citep{hafner2023mastering}, which learns a latent recurrent state-space model (RSSM). 
For sparse IDM, we employ the model backbones from CLAM~\citep{liang2025clam} and further impose sparsity on this latent action space. Details are given in Appendix~\ref{sec:idm_architecture}. 

\ifnips
\else
\ifnips
\subsection{Robustness of World Action Verification}
\else
\subsubsection{Robustness of World Action Verification}
\fi

\textbf{Setup.}
To evaluate the robustness of WAV, similar to the evaluation in MiniGrid, we compute the Spearman’s rank correlation coefficient~\cite{spearman1961proof}
between the data selection scores of each method and
those of the Oracle method. We use 100 samples that are held out from the world model training data.

\textbf{Results.} As shown in Figure~\ref{fig:corr_robo}, our verification scores more faithfully reflect the true (oracle) difficulty ranking of samples, verifying the robustness of WAV. 

\subsection{Full Results on World Model Prediction}
Table~\ref{tab:wm_error_200} reports the prediction error under the 200-sample budget. 
Ours consistently achieves the lowest error among all non-Oracle methods across the six tasks, suggesting that inverse-verification-based data selection provides a more effective adaptation signal than uncertainty-, progress-, or novelty-based acquisition. 
The improvement is especially clear when compared with non-IDM baselines, where Ours is significantly better across all tasks under a Wilcoxon signed-rank test~\citep{conover1999practical} at the $5\%$ level. 
Compared with Vanilla IDM, Ours also further reduces prediction error, indicating that the verifier is not only useful as an inverse-dynamics filter, but also helps select transitions that are more informative for action-conditioned world-model adaptation. 
\begin{table*}[h]
\centering
\caption{World model prediction error with 200 training samples. 
Results are reported as mean $\pm$ standard error. 
The best non-Oracle result for each task is highlighted in light yellow. 
$^\dagger$ indicates that the best non-Oracle method is likely significantly better than the second-best non-Oracle/non-IDM method at the 5\% level.}
\label{tab:wm_error_200}
\setlength{\tabcolsep}{5.5pt}
\renewcommand{\arraystretch}{1.12}
\resizebox{\textwidth}{!}{
\begin{tabular}{lcccccc}
\toprule
\textbf{Method} 
& \textbf{Can} 
& \textbf{Lift} 
& \textbf{Square} 
& \textbf{PokeCube} 
& \textbf{PullCube} 
& \textbf{LiftPeg} \\
\midrule
RND 
& $26.5 \pm 0.7$
& $23.3 \pm 0.8$
& $39.7 \pm 0.8$
& $30.1 \pm 0.9$
& $24.6 \pm 0.8$
& $23.4 \pm 0.7$ \\
Uncertainty 
& $22.9 \pm 0.6$
& $20.7 \pm 0.6$
& $33.0 \pm 0.7$
& $25.2 \pm 0.7$
& $19.9 \pm 0.5$
& $19.4 \pm 0.6$ \\
Progress 
& $22.0 \pm 0.6$
& $19.6 \pm 0.6$
& $35.1 \pm 0.7$
& $27.3 \pm 0.7$
& $21.1 \pm 0.6$
& $20.5 \pm 0.8$ \\
Vanilla IDM 
& $20.2 \pm 0.2$
& $18.3 \pm 0.5$
& $29.8 \pm 0.7$
& $23.4 \pm 0.6$
& $18.2 \pm 0.5$
& $17.5 \pm 0.5$ \\
Ours 
& \cellcolor{bestyellow}$18.2 \pm 0.5^{\dagger}$
& \cellcolor{bestyellow}$16.4 \pm 0.4^{\dagger}$
& \cellcolor{bestyellow}$26.0 \pm 0.3^{\dagger}$
& \cellcolor{bestyellow}$22.4 \pm 0.5^{\dagger}$
& \cellcolor{bestyellow}$17.1 \pm 0.4^{\dagger}$
& \cellcolor{bestyellow}$16.8 \pm 0.5^{\dagger}$ \\
\bottomrule
\end{tabular}
}
\end{table*}
\subsection{Full Results on Policy Learning}
\label{sec:full_results_policy}
We provide the full downstream policy learning results in Figure~\ref{fig:rewards}. 
For each learned world model, we refine the same base diffusion policy through imagination-based search following the SAILOR protocol~\citep{jain2025smoothseaskilledsailor}, and report the resulting task reward after fine-tuning with a fixed data budget of 1,000 trajectories. 
Across RoboMimic and ManiSkill tasks, policies refined with WAV-based world models achieve higher rewards than those refined with baseline world models, and are second only to the oracle model that uses privileged ground-truth actions for sample selection. 
The improvement is most pronounced on tasks with ambiguous or contact-rich dynamics, such as \texttt{Can}, \texttt{Square}, and \texttt{PokeCube}, where accurate action-conditioned latent dynamics are important for effective imagination-based policy refinement. 
In contrast, simpler tasks such as \texttt{Lift} show smaller gaps across methods, suggesting that downstream policy gains are most sensitive to world-model quality when the task requires resolving more complex object interactions.
\begin{figure}[t]
    \centering
    \includegraphics[width=1.0\linewidth]{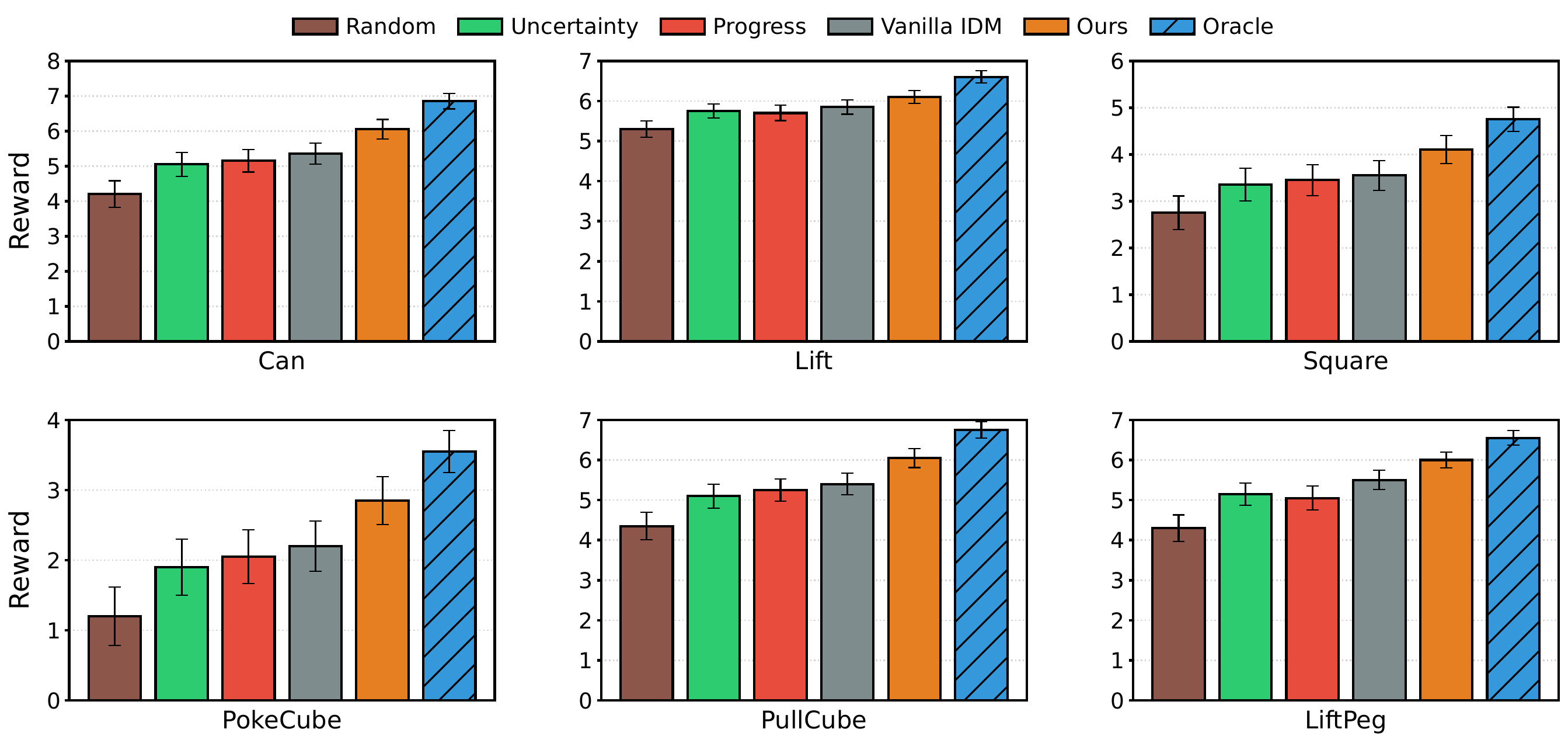}
    \caption{\textbf{Downstream policy performance on RoboMimic and ManiSkill using learned world models.}  Error bars denote the standard error over 3 seeds.}
    \label{fig:rewards}
\end{figure}


\subsection{OOD Setups}
We construct two out-of-distribution variants  (A visualization on RoboMimic \texttt{Can} task: Figure~\ref{fig:ood_setting}). 
The first variant introduces visual appearance shifts by changing nuisance rendering factors, including the background and embodiment color, while keeping the underlying task dynamics unchanged. 
The second variant introduces object and interaction shifts by adding multiple objects and collecting demonstrations from diffusion policy checkpoints with different training progress, resulting in a mixture of expert-like, medium-quality, and suboptimal behaviors. 
For each OOD variant, we initialize from the world model trained on the original \texttt{Can} environment and adapt it using only 200 target-domain trajectories. 
We evaluate both 32-step next-observation prediction error on held-out OOD trajectories and downstream reward after imagination-based policy refinement under the same SAILOR-style protocol used in the in-distribution robotic experiments.
\begin{figure}
    \centering
    \includegraphics[width=\linewidth]{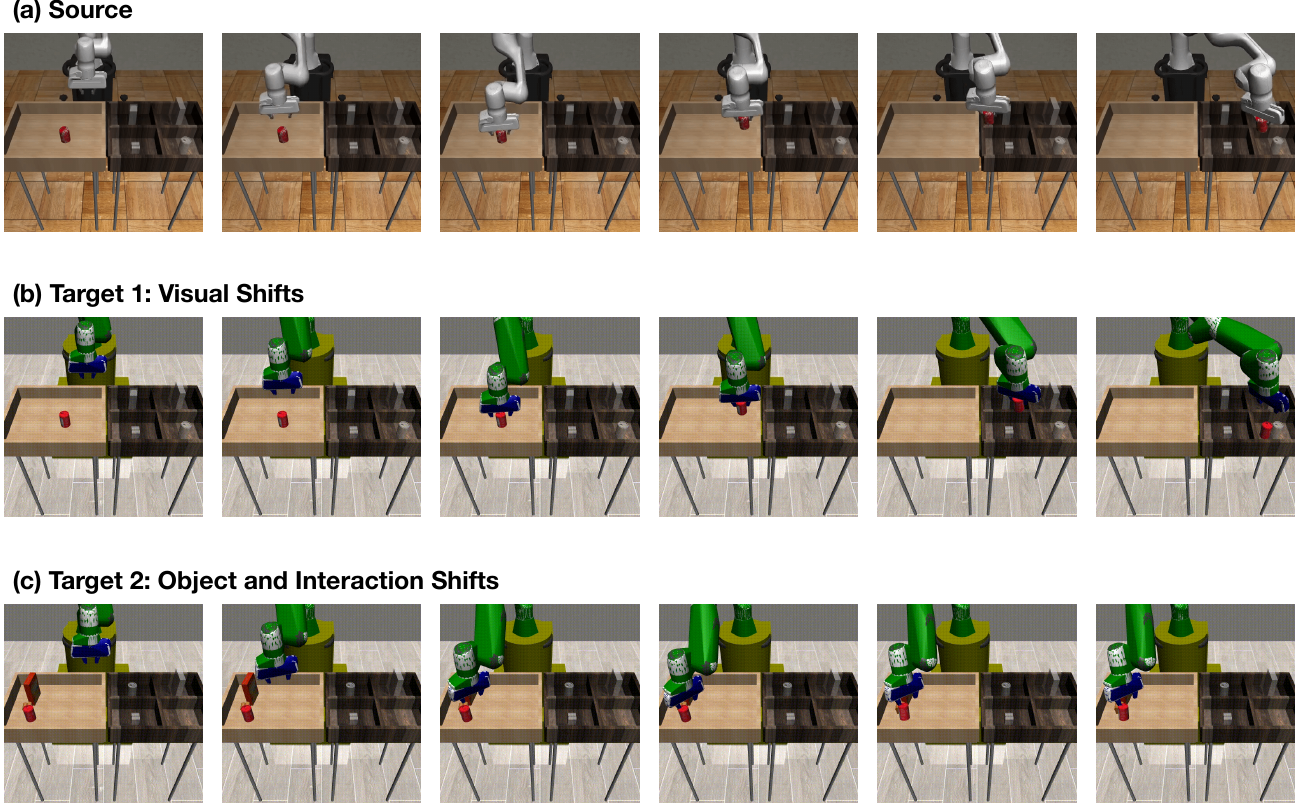}
\caption{
\textbf{Visualization of the OOD adaptation settings.}
Starting from the original RoboMimic \texttt{Can} environment, we construct two categories of target-domain shifts.
The first introduces visual shifts, including modified backgrounds and embodiment colors, while preserving the same task dynamics.
The second introduces more challenging object and interaction shifts, including scenes with multiple objects and demonstrations collected from policies with different levels of optimality.
These settings test whether the learned world model can adapt to both visual nuisance changes and harder dynamics-relevant distribution shifts with limited target-domain data.
}
\label{fig:ood_setting}
\end{figure}

\fi

\subsubsection{Effectiveness for World-Model Learning}
\label{sec:id_wm_policy}

To address \textbf{\textit{RQ3}}, we first evaluate whether WAV improves world-model learning under different data budgets in simulated robotic manipulation tasks. 


\textbf{Setup.} 
We warm-start each world model using \{200, 400, 600, 800, 1000\} trajectories (each containing approximately 200 episodic samples) and then fine-tune it for another 200 epochs with the self-improving loop. 
We report the average observation MSE over 32 predicted frames after 2 exploration rounds, averaged over 3 seeds. 


\textbf{Results.} 
Figure~\ref{fig:wm_pred_robots} shows the predictions of the world-model across different data budgets. Results under the 200-sample budget, together with standard errors, are reported in Appendix Table~\ref{tab:wm_error_200}.
WAV consistently outperforms all baselines, with especially large gains in the low-data regime. 
Sparse inverse dynamics models also consistently outperform dense variants on object manipulation tasks, supporting our hypothesis that sparsity improves inverse verification under limited data. 
The full downstream policy evaluation in Appendix~\ref{sec:full_results_policy} and Figure~\ref{fig:rewards} follow the same trend: world models improved by WAV support stronger imagination-based policy refinement than baseline world models, suggesting that the prediction gains correspond to more useful latent dynamics for control.

\subsubsection{OOD Adaptation and Downstream Policy Learning}
\label{sec:ood_generalization}

To address \textbf{\textit{RQ4}} and \textbf{\textit{RQ5}}, we evaluate whether WAV improves adaptation to out-of-distribution robotic manipulation settings and whether it supports better downstream policy refinement. 

\textbf{Setup.}
We use RoboMimic \texttt{Can} as the base environment and construct two OOD variants (visualized in Appendix Figure~\ref{fig:ood_setting}). 
The first introduces \textit{visual shifts}, changing nuisance factors such as background and embodiment color while keeping the task dynamics unchanged. 
The second introduces \textit{object and interaction shifts}, increasing task complexity with multiple objects and demonstrations of mixed optimality. 
These demonstrations are collected from diffusion policy checkpoints trained for different numbers of steps, producing expert-like, medium-quality, and suboptimal behaviors.

For each OOD variant, we start from the world model trained on the original RoboMimic \texttt{Can} data and adapt it using only 200 target-domain trajectories. 
We evaluate both world-model quality, measured by MSE on held-out OOD trajectories (with $256 \times 256$ resolutions), and downstream policy performance, measured by reward after imagination-based policy refinement under the SAILOR-based protocol~\citep{jain2025smoothseaskilledsailor}.
\begin{figure}[t]
    \centering
    \includegraphics[width=1.0\linewidth]{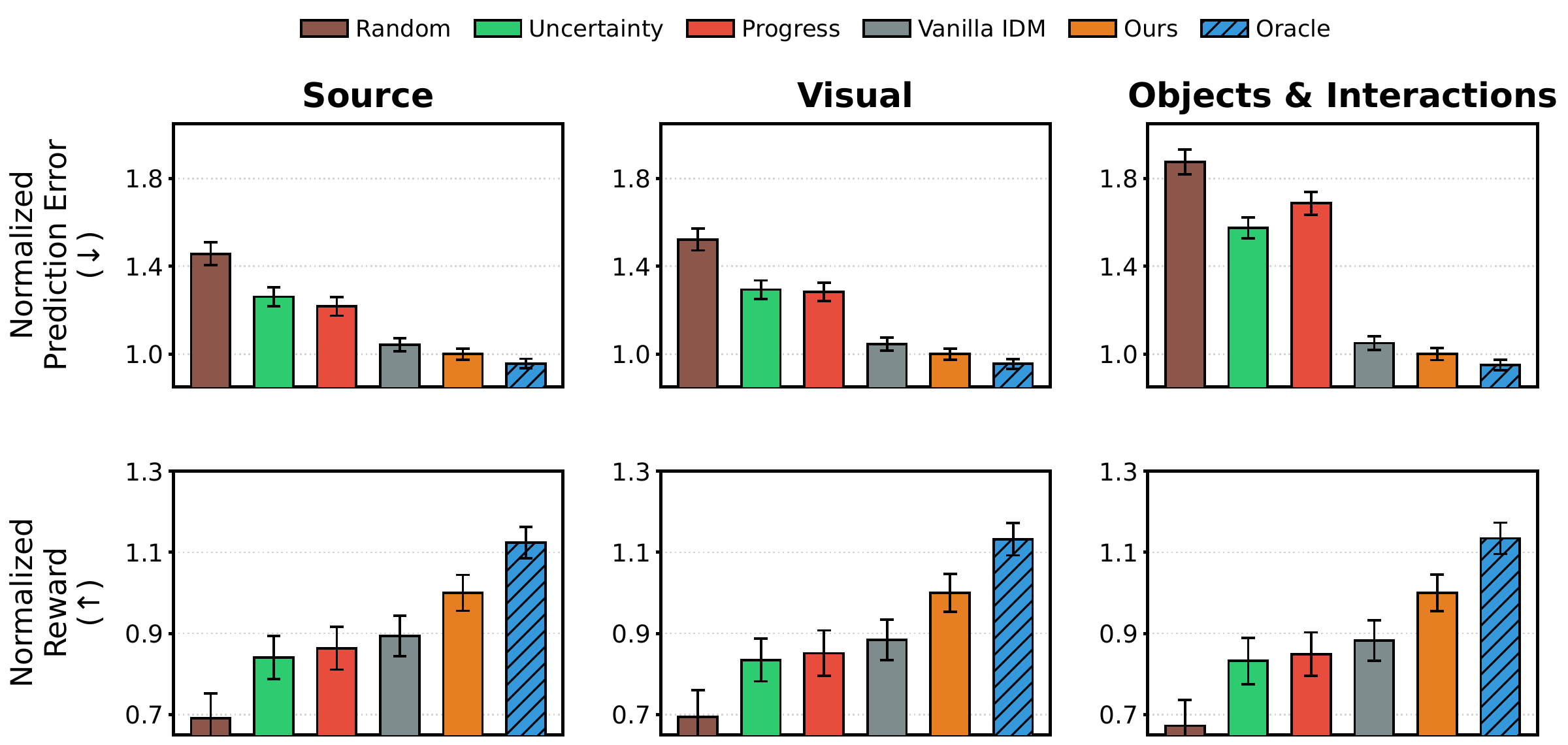}
    \caption{
    \textbf{OOD adaptation results on RoboMimic.}
    We evaluate world-model adaptation (normalized prediction error and downstream reward) under visual shifts and objects/interaction shifts.
     Error bars denote the standard error over 3 seeds.}
    \label{fig:ood_results}
\end{figure}


\textbf{Results.}
Figure~\ref{fig:ood_results} shows that WAV achieves stronger OOD adaptation than the baselines under both types of shift. 
Under visual shifts, where the dynamics are unchanged, but nuisance appearance factors vary, WAV better preserves action-conditioned dynamics after adaptation and achieves lower prediction error with the same 200 target trajectories. 
The corresponding reward results show that this improvement is not only a prediction-level gain: policies refined with WAV-adapted world models also achieve stronger downstream performance. Full in-distribution reward results across all environments are provided in Appendix~\ref{sec:full_results_policy} and Figure~\ref{fig:rewards}.

The gains become larger under object and interaction shifts, where multiple objects and mixed-optimality demonstrations introduce diverse contacts, ambiguous interactions, and under-explored transition regions. 
In this setting, heuristic acquisition strategies may either overemphasize visually novel but less informative transitions or miss rare interaction-rich cases that are critical for learning accurate dynamics. 
By focusing adaptation on transitions that are difficult for the current world model but still action-reachable, WAV provides a more data-efficient mechanism for adapting world models to under-explored target regions. 
Overall, the combined improvements in MSE and reward suggest that WAV learns OOD dynamics that are not only more predictive, but also more useful for downstream imagination-based policy refinement, achieving an approximately $22\%$ improvement on novel environments with new objects and interactions.

\section{Related Work}

{\bf Exploration for World Models.}
Exploration for world models asks which interactions most improve a learned dynamics model. Classical coverage-based objectives use counts or density surrogates~\citep{bellemare2016unifyingcountbasedexplorationintrinsic,ostrovski2017countbasedexplorationneuraldensity,burda2018explorationrandomnetworkdistillation,gxchen2025efficientexplorationdiscriminativeworld}, while model-based variants estimate value through uncertainty or disagreement~\citep{houthooft2017vimevariationalinformationmaximizing,shyam2019modelbasedactiveexploration,sekar2020planning}, learning progress~\citep{graves2017automatedcurriculumlearningneural,achiam2017surprisebasedintrinsicmotivationdeep,kim2020activeworldmodellearning}, prediction error and curiosity~\citep{pathak2017curiositydrivenexplorationselfsupervisedprediction,haber2018learningplayintrinsicallymotivatedselfaware,stadie2015incentivizingexplorationreinforcementlearning,du2021curious}, or goal discovery with learned world models~\citep{mendonca2021discovering,hu2023planning,zheng2024can,liu2024single}. These signals are useful but often come from the current action-conditioned world model, making them unreliable in under-explored regimes where verification is most needed. Related work also exploits unlabeled prior data for optimistic exploration~\citep{Rashid2020Optimistic, li2023accelerating}, unsupervised skill discovery~\citep{wilcoxson2024leveraging, wang2024skild}, and mutual-information objectives~\citep{eysenbach2018diversity, zheng2024can}, but typically aims to learn exploration behaviors. We instead verify informative transitions through plausible future states and sparse inverse-dynamics reachability to directly enhance action-conditioned world models.

{\bf World Action Models.}
World action models (WAMs) have recently made substantial progress in policy learning by jointly leveraging action-free internet video and action-labeled robot data. One line jointly predicts future video and actions within a unified model, allowing large-scale action-free data to improve representations without changing the downstream policy interface~\citep{cheang2024gr,rigter2024avid,huang2025vid2world,zhu2025unified,li2025unified,bi2025motus,wu2024ivideogpt}. A second line performs visual planning or foresight generation and conditions action prediction on future frames, shifting more of the planning burden into semantic image or video space while keeping low-level control at the action level~\citep{du2023learning,black2023zero,chen2025large,jang2025dreamgen,vuong2026world2act,hu2025video,lv2025f1,cai2026internvla,ye2026world,ye2026gigaworld, li2026causal, yuan2026fast}. Our method is formally closer to the second family: like these approaches, it combines future-state generation with inverse action prediction. The difference is that we use this structure not as a policy model, but to expose an accuracy advantage of inverse verification over world models and to repurpose a WAM-like pipeline as a verifier for action-conditioned world models.

\section{Conclusion}
\label{sec:conclusion}

In this work, we identified an essential asymmetry between forward and inverse dynamics: inferring which action caused a plausible transition can be substantially easier than predicting its full outcome. Building on this insight, we introduced \methodname{}, a self-improving framework that exploits cycle consistency among a diverse subgoal generator, a sparse inverse model, and a forward world model to gather informative interactions. Across MiniGrid, RoboMimic, and ManiSkill, our method enables $2\times$ greater sample efficiency in exploration and improves downstream policy performance by 22\%.

\section*{Acknowledgement}

We thank the members of the IRIS Lab for valuable feedback and discussions. We also thank Xiangcheng Zhang for help with the simulation setup. This work was supported in part by the Robotics and AI Institute, DARPA, ONR, CIFAR, SNSF, Schmidt Science, NSF, NIH, and the AI Institute for Societal Decision Making.

\bibliography{references}


\clearpage
\newpage
\appendix
\section*{Appendix Contents}
\label{appendix}
\addtocontents{toc}{\protect\setcounter{tocdepth}{1}}
\vspace{0em}
\makeatletter
\@starttoc{toc}
\makeatother
\vspace{2em}

\ifnips
\section{Additional Experiments}
\label{app:experiment}
\ifnips

\begin{figure}[t]
    \centering
\includegraphics[width=.92\linewidth]{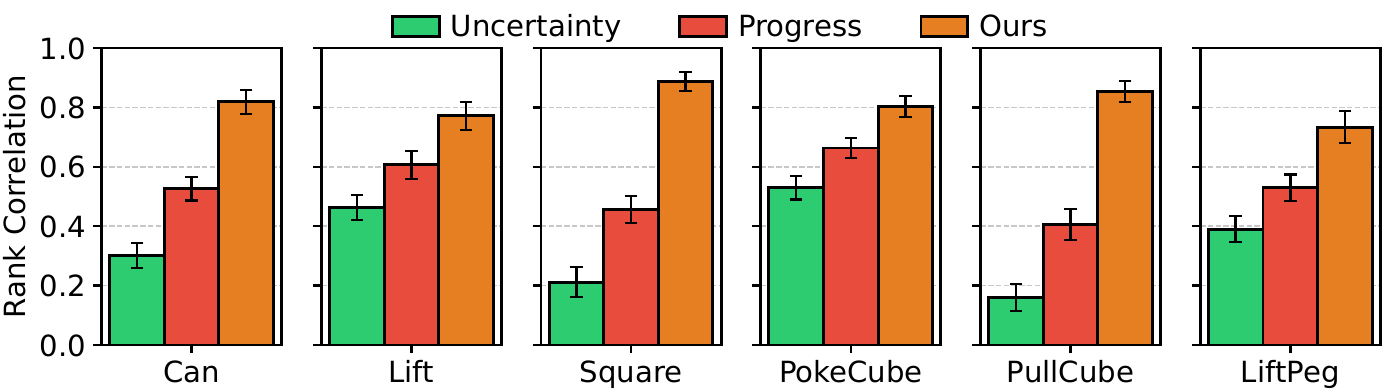}
    \caption{\textbf{Robustness verification of WAV on Robomimic and ManiSkill. }Correlation with Oracle ranking. We evaluate how well each method orders informative samples by computing Spearman rank correlations between the method’s assigned scores and Oracle scores on RoboMimic and ManiSkill environments. Higher correlation indicates closer agreement with the Oracle’s ranking.}
    \label{fig:corr_robo}
\end{figure}

\else
\fi

\else
\fi

\section{Additional Related Work}
\label{app:related}
\paragraph{World Models for Robotics.}

Model-based reinforcement learning (MBRL) learns predictive environment dynamics for planning and policy improvement. Early approaches focused on probabilistic, data-efficient control~\citep{sutton90integratedarchitectures,deisenroth2011pilco}, while modern deep MBRL combines expressive dynamics models with planning and imagined rollouts in off-policy learning~\citep{chua2018deep,janner2019trust}. By reasoning over predicted futures, these methods can be highly sample-efficient and effective for control. However, they are often tied to specific tasks or policy distributions, which limits their robustness under distribution shifts. To mitigate this, prior work has explored model ensembles~\citep{buckman2018sample, janner2019trust}, conservative optimization \citep{yu2021combo, kolev2024efficient}, and online fine-tuning initialized from offline priors~\citep{rafailov2023moto, feng2023finetuning}. Nevertheless, these approaches still remain limited in their scalability to high-dimensional perception, diverse interaction regimes, and broad generalization.

More recently, general-purpose world models that learn predictive representations from large and diverse sequential data have been developed. One line of work focuses on \textit{latent world models}, which learn compact action-conditioned dynamics from high-dimensional observations and perform prediction, imagination, and control in a learned latent space~\citep{ha2018recurrent,hafner2019learning, Kaiser2020Model, hafner2019dream,hafner2023mastering, zhou2025dinowmworldmodelspretrained, garrido2024learning, maes2026leworldmodel}. A second line is more \textit{planning- and control-centric}, with objectives that are directly for decision making and policy improvement~\citep{schrittwieser2020mastering,hansen2022temporal,hansen2023td}. A third line studies \textit{pixel-based} world models trained on large-scale robotics or even internet video data~\citep{wu2024ivideogpt,rigter2024avid, chen2025large}, either by combining video generation with inverse dynamics models~\citep{chi2025wow} or by directly learning action-conditioned dynamics~\citep{guo2025ctrlworldcontrollablegenerativeworld,hafner2025training,gao2026dreamdojo}, with the goal of producing visually realistic yet controllable futures.
A key advantage of these models is that they can leverage internet-scale data and increasingly scalable model sizes to acquire broad predictive priors. Despite their strong generative capacity, these models still face important challenges in physical consistency and action following~\citep{shang2026worldarena, mei2026video}, especially when deployed for robotics control. Our work addresses this through targeted exploration, which actively collects informative interactions with verification to improve the robustness of action-conditioned world models.

\paragraph{Inverse Dynamics on Videos.}
A common strategy for leveraging action-free observations is to infer actions from state transitions. Existing work uses inverse dynamics in different roles in world model and policy learning. As \emph{policies}, IDMs convert predicted or planned future observations into executable actions, as in visual-planning and foresight-based systems~\citep{du2023learning,black2023zero,chen2025igor,tian2025predictive,hu2025video,lv2025f1,luo2025self,cai2026internvla,ye2026world}. As \emph{regularizers}, inverse objectives have long been used to shape self-supervised representations and dynamics features~\citep{agrawal2016learning,pathak2017curiositydrivenexplorationselfsupervisedprediction}. As \emph{labelers}, inverse models can impute missing actions from state-only or video demonstrations~\citep{torabi2018behavioral,yang2019imitation,baker2022video,jang2025dreamgen}. As \emph{verifiers}, they can test whether a transition is action-consistent; conceptually, our approach is closest to this line and to RLIR~\citep{ye2025reinforcement}. However, our verifier differs in two key ways: it uses a reverse cycle anchored on plausible future states, and it checks reachability with a sparse action-relevant inverse model rather than dense full-state generation.

\paragraph{Self-improvement via verification.}
A parallel line of work studies self-improvement loops driven by internally generated feedback, especially in language models and self-play agents.
In alignment, solver-verifier pipelines optimize policies against learned feedback signals~\citep{ouyang2022training,bai2022constitutional,rafailov2023direct}; explicit verifiers have also been used to filter and rerank candidate solutions in domains like math reasoning \citep{cobbe2021training,zhang2024generative}.
Complementary approaches bootstrap supervision by generating candidate solutions and selectively adding verified/high-quality outputs back into training~\citep{zelikman2022star,wang2023self}, as well as test-time self-critique and refinement loops \citep{madaan2023self,shinn2023reflexion}.
Finally, self-play provides a principled mechanism to create increasingly challenging data and feedback~\citep{silver2017mastering,huang2025r,zhang2025path,liu2025spiral,liu2025spice,wang2025improving}.
Our work shares the idea of using a verifiable internal signal to turn novel experience into training data.
While self-improvement is well-studied for language models, \emph{few works explore it for world models}, where verification is harder due to continuous dynamics and absent symbolic ground truth.
\methodacro addresses this gap with a video-prior subgoal generator and sparse inverse-dynamics cycle consistency as a reachability verifier.

\section{Additional Theoretical Analysis}
\label{subsec:identifiability}

In this section, we analyze the exact conditions under which the sparse inverse verifier can generalize beyond the limited training distribution of labeled transitions.
To analyze this, we model the observed state $s^t$ as arising from a latent vector $\zz^{t}=(\zz_1^{t},\dots,\zz_k^{t})$.
The learned mask $M$ in \eqref{eq:sparse_inverse_def} selects an action-relevant block $\cS$ of this latent space; intuitively, $\cS$ captures agent-centric variables (\eg, proprioception or end-effector motion) and is largely insulated from the rest of the scene.
The sparse inverse model $h_\psi$ from \cref{sec:method} thus operates on $(\zz_{\cS}^{t},\zz_{\cS}^{t+1})$; we write the verifier as
$\hat{\aa}^{t}=h_{\psi}(\hat{\zz}^{t}_{\cS},\hat{\zz}^{t+1}_{\cS})$, where $\hat{\zz}^{t}$ denotes the encoder's latent estimate of $\zz^{t}$.

Let $P_{\text{seed}}$ denote the distribution induced by $\mathcal{D}_{\mathrm{act}}$; we call a state--action pair \emph{out-of-support} (OOS) when $(\zz^{t},\aa^{t})\notin\supp(P_{\text{seed}})$.
The key structural condition is the presence of a \emph{generation--verification gap}: the full pair $(\zz^{t}, \aa^{t})$ may be OOS, while the restricted pair $(\zz^{t}_{\cS}, \aa^{t})$ remains on support. This captures the regime in which scene-level consequences are novel, but the agent-side motion that encodes the action is still familiar.

\begin{proposition}[Informal]
\label{prop:oos_id}
Assume there exists an identifiable verification subset $\cS$ such that: (i) $\zz_{\cS}^{t+1}$ depends only on $(\zz_{\cS}^{t},\aa^{t})$ and not on the rest of the scene; (ii) $(\zz_{\cS}^{t},\aa^{t})$ stays on-support even when $(\zz^{t},\aa^{t})$ is OOS; and (iii) the action is identifiable from the subset transition $(\zz_{\cS}^{t},\zz_{\cS}^{t+1})$.
Then an inverse model trained on the seed data can recover the correct action from $(\hat{\zz}_{\cS}^{t},\hat{\zz}_{\cS}^{t+1})$ on such compositional OOS transitions.
Consequently, the forward--inverse mismatch used by \methodacro localizes forward-model error rather than action-label ambiguity.
\end{proposition}

{\bf Interpretation.}
Proposition~\ref{prop:oos_id} guarantees that the sparse inverse model $h_\psi$ produces correct pseudo-labels whenever the agent's own motion pattern (\eg, joint-angle trajectories) was seen during training, even if the full scene transition is novel.
This is a strictly weaker requirement than asking the \emph{full} transition to be on-support, which is what a dense forward model or a full-observation inverse model would need.
The direct consequence for the self-improving cycle in \cref{subsec:selfimprove} is that the discrepancy $\ell(\tilde{s}^{t+1}, \hat{s}^{t+1})$ between the subgoal and the forward rollout reflects genuine world-model error rather than action-label noise, so each exploration round adds trustworthy data that expands the world model's effective coverage.
Appendix~\ref{app:identifiability} formalizes this claim and shows how the verification subset $\cS$ can be identified from observations.

\section{Additional Discussions}
\label{app:discussion}
Although our primary focus is verification-guided exploration for acquiring new interactions, the proposed \methodacro may also be useful in other settings that benefit from robust error estimation, such as test-time scaling~\citep{nakamoto2024steering,liu2025bidirectional,kwok2025robomonkey} and offline data curation~\cite{hejna2025robot,chen2025curating,agia2025cupid}. Nevertheless, the current instantiation of our method requires three inference passes, making it computationally more expensive than prior exploration methods. Improving its efficiency through shared intermediate representations~\citep{zhu2025unified,li2025unified} or adaptive computation mechanisms~\citep{yang2026dynamic,zhang2025inference,lin2026plan} is an important direction for enabling more affordable real-time deployment.

More broadly, our results suggest that the forward--inverse asymmetry may be especially pronounced in high-dimensional, uncertain environments. However, fully self-improving world models in such complex settings remain elusive. Unlike language models, which can already improve from purely synthetic data on some reasoning tasks, our method still relies critically on additional environment feedback to correct action-conditioned prediction errors. Reducing this reliance will likely require substantially stronger verification mechanisms~\citep{lifshitz2025multi,kwok2026scaling}, likely scaffolding on more capable pretrained models. Extending our \methodname{} to incorporate richer generative priors~\citep{chen2025large,gao2026dreamdojo} and more expressive inverse models~\citep{tian2025predictive,ye2026world} in longer-horizon embodied tasks~\citep{liu2023liberobenchmarkingknowledgetransfer,nasiriany2024robocasa,bu2025agibot,dai2026robomme} can be promising directions for future work.

\section{Additional Implementation Details}

\subsection{Compute Resource}
For the MiniGrid experiments, we utilized $1\times$ NVIDIA RTX 4090 GPU, with each full experimental run requiring approximately $2$ to $3$ GPU hours. This duration encompasses world model training, verifier training and active learning iterations. For the robotic manipulation experiments, we used either $6\times$ NVIDIA L40S GPUs or $3\times$ NVIDIA H100 GPUs across all experiments. On average, WAV required approximately 40 GPU hours per robotic environment, while the baselines required approximately 36 GPU hours under comparable settings. These include world-model training, verifier training, and downstream imagination-based policy refinement.

\subsection{MiniGrid Setting}
\label{sec:minigrid_setting}

We conduct experiments in the MiniGrid simulation environment, using \texttt{EmptyEnv} with three object types: key, ball, and box, each of which can be either red or blue.
The agent has seven discrete actions: \textit{turn left, turn right, move forward, pick up, drop, toggle, and swap.} The behavior of the toggle action is object-dependent: for keys and balls, it switches the object’s color; for boxes, it acts as an exchange mechanism, swapping the item currently held by the agent with the item inside the box (or placing the held item inside if the box is empty).
The swap action is not part of the original EmptyEnv. We define it as exchanging the object in front of the agent with the object it is carrying.
Based on this setup, we designed three tasks in EmptyEnv, the details of which can be found in ~\cref{sec:task_definitions}.

\subsubsection{Task Definitions}
\label{sec:task_definitions}
To evaluate the agent's ability to handle long-horizon dependencies and compositional logic, we design three complex tasks in the MiniGrid environment. Each task requires the agent to manipulate objects (Key, Ball, Box) based on their attributes (Red, Blue).

\begin{itemize} \item \textbf{Task 1: Key Delivery.} The agent must: (1) locate a key, (2) change its color to match the target box, (3) place the key inside the box, (4) swap the box with a ball, (5) adjust the ball's color to match the box, and (6) reach the goal. 

\item \textbf{Task 2: Ball Delivery.} 
This is a structural mirror of Task 1 but swaps the roles of the key and the ball. The agent must place the ball inside the box before manipulating the key and reaching the goal.

\item \textbf{Task 3: Object Matching.} 
The agent must: (1) identify the reference color of the box, (2) locate the key and ball, (3) synchronize the key's color with the box, (4) synchronize the ball's color with the box, (5) place both the key and the ball around the box, and (6) reach the goal.
\end{itemize}

\subsubsection{Dataset Composition.}

\textbf{Random Play Dataset.}
We construct random play datasets based on the \texttt{EmptyEnv}, where objects can be freely placed. 
To study state complexity, we vary the number of objects $\{6, 8, 10, 12, 14\}$ and collect trajectories using random policies, resulting in environments with increasingly complex object configurations. 
In this setting, only the environment with 6 objects is used to construct both the training and test sets, while environments with higher object counts are used exclusively for testing, enabling controlled evaluation of generalization to more complex scenes.

To study environmental stochasticity, we vary the number of noisy floor tiles $\{0,1,2,3,4\}$, whose colors change randomly at every step. For each noise level, we collect trajectories with random policies and construct both training and test sets, allowing us to evaluate robustness under different levels of environmental noise.

\textbf{Exploration Pool.}
We collect a total of 56{,}273 transitions across the three tasks defined in~\cref{sec:task_definitions}, covering diverse state–action–next-state tuples. 
More than half of the collected data (28{,}000 transitions) is used as an unlabeled pre-training set to train the video model without action annotations. 
The remaining 28{,}273 transitions form the exploration pool, from which different acquisition strategies iteratively select informative samples. 
We additionally construct an action-balanced test set of 10{,}368 transitions for evaluation.

\textbf{Compositional OOD Generalization.}
To evaluate compositional out-of-distribution generalization, we partition action–object–color combinations as summarized in Table~\ref{tab:action_object_oos}. 
During training, models are exposed only to a restricted subset of combinations (e.g., \emph{pick up blue keys}), while evaluation is conducted on held-out compositions involving unseen combinations (e.g., \emph{pick up blue balls}). This setup tests whether the model can generalize compositionally beyond observed training distributions.

\begin{table}[t]
\centering
\caption{Statistics of the collected dataset for exploration. The data is split into an unlabeled pre-training set for learning the video model, an exploration pool for sample acquisition, and an action-balanced test set for evaluation.}
\label{tab:dataset_stats}
\begin{tabular}{lr}
\hline
\textbf{Category} & \textbf{Count (Transitions)} \\ \hline
Total Collected & 66,641 \\
Unlabeled Pre-training Set & 28,000 \\
Exploration Pool & 28,273 \\
Test Set (Action-balanced) & 10,368 \\ \hline
\end{tabular}
\end{table}
\providecommand{\seen}{\textcolor{green!60!black}{$\checkmark$}}
\providecommand{\oos}{\textcolor{orange!85!black}{$\star$}}
\providecommand{\absent}{\textcolor{gray!70}{$\times$}}

\begin{table}[t]
\centering
\small
\setlength{\tabcolsep}{4pt}
\caption{Action--object composition coverage in the training and OOS test sets.
\seen\ denotes compositions seen during training,
\oos\ denotes OOS-only compositions evaluated at test time,
and \absent\ denotes combinations absent from both sets.}
\begin{tabular}{lcccc}
\toprule
Action $\backslash$ Object & red key & blue key & red ball & blue ball \\
\midrule
Pick up              & \absent & \seen & \absent & \oos \\
Drop                 & \oos    & \absent & \seen & \absent \\
Toggle               & \oos    & \seen & \seen & \oos \\
Swap (box for \dots) & \oos    & \absent & \seen & \absent \\
\bottomrule
\end{tabular}
\label{tab:action_object_oos}
\end{table}

\subsubsection{Evaluation Metrics.}
\label{sec:minigrid_metrics}

\paragraph{Prediction Accuracy.}
We first report Dynamics Accuracy, which measures prediction accuracy only over elements that undergo temporal changes, including both visual grid cells and internal agent attributes (e.g., carried status), while masking out invariant background regions. By focusing on these dynamic components, Dynamics Accuracy mitigates metric inflation caused by static background dominance and better reflects the model’s ability to capture action-driven dynamics.

In exploration experiments, we further evaluate all methods using the world model’s next-state prediction loss on the held-out test set. We utilize prediction loss in this context because it provides a more sensitive signal of training stability and convergence behavior during exploration, whereas Dynamics Accuracy is better suited for interpreting final predictive performance.

\paragraph{Ranking Quality.}
To assess the quality of data selection, we compute the rank correlation between method-assigned scores and Oracle scores using Spearman's rank correlation coefficient ($\rho$)~\cite{spearman1961proof} and Kendall's rank correlation coefficient ($\tau$)~\cite{kendall1938new}.
Given a set of samples with scores $\{s_i\}$ and corresponding Oracle scores $\{o_i\}$, the Spearman correlation is defined as
\begin{equation}
\rho = 1 - \frac{6 \sum_i (r_i - q_i)^2}{n(n^2 - 1)},
\end{equation}
where $r_i$ and $q_i$ denote the ranks of $s_i$ and $o_i$, respectively. 

Kendall’s $\tau$ measures the consistency of pairwise orderings:
\begin{equation}
\tau = \frac{N_c - N_d}{\frac{1}{2} n(n-1)},
\end{equation}
where $N_c$ and $N_d$ are the numbers of concordant and discordant pairs. Higher values indicate stronger agreement with the Oracle ranking.

\paragraph{Action Following Score.}
\label{sec:action_following}
In addition to prediction accuracy and ranking quality, we evaluate whether the learned world model can capture action-dependent dynamics. To this end, we introduce the \textit{Action Following Score} (AFS), which measures how well the model preserves distinctions between different actions in its predicted future states.

Given an initial state $s_0$ sampled from the test set and a set of candidate actions $\{a_i\}_{i=1}^N$, we obtain predicted next states $\hat{s}_i = f_\theta(s_0, a_i)$ from the learned world model, and corresponding ground-truth next states $s_i$ from the simulator. We define a difference function $\mathrm{Diff}(\cdot, \cdot)$ that counts the number of grid cells with different values between two states. The Action Following Score is then defined as
\begin{equation}
\text{AFS}(s_0) = 
\frac{
\sum_{i < j} \mathrm{Diff}(\hat{s}_i, \hat{s}_j)
}{
\sum_{i < j} \mathrm{Diff}(s_i, s_j)
}.
\end{equation}
We report the final score by averaging $\text{AFS}(s_0)$ over initial states sampled from the test set.

Intuitively, the denominator measures the true diversity induced by different actions, while the numerator reflects how much of this diversity is preserved by the model. A higher AFS indicates that the model produces more distinguishable predictions across actions, suggesting stronger action-conditioned modeling.

\subsubsection{Models in MiniGrid}

\textbf{World Model.}
We employ a physics-aligned architecture that preserves spatial structure via coordinate-aware convolutions.
The model incorporates a \textbf{supervised vector-quantized bottleneck}~\cite{oord2018neuraldiscreterepresentationlearning}, which explicitly maps latent codes to discrete actions and conditions a residual dynamics engine through Feature-wise Linear Modulation (FiLM)~\cite{perez2017filmvisualreasoninggeneral}, promoting object persistence and physical consistency.

\textbf{Inverse Dynamics Models (IDM).} To verify the robustness of sparse IDM, we compare two IDMs on the OOD test set:
\begin{itemize}
    \item \textbf{Vanilla IDM:} Takes the entire observation frame and the agent's proprioceptive state as input.
    \item \textbf{Sparse IDM:} Built upon the vanilla IDM by applying a learnable feature mask to the input, which selectively filters out irrelevant information and yields a sparse representation. The mask is learned automatically during training, encouraging the model to focus on the most informative local features.
\end{itemize}

Architecturally, both variants share a convolutional encoder that extracts spatial features from observations, followed by a projection layer that integrates object-centric attributes. Action prediction is performed by jointly reasoning over embeddings of consecutive frames, augmented with explicit geometric cues including relative position and direction changes.

\subsubsection{Exploration Methods in MiniGrid}

The \textbf{Oracle} strategy used in \cref{sec:experiments} selects samples with the highest prediction loss, corresponding to the hardest transitions under the current world model.
To better understand the role of sample difficulty during exploration, we further introduce two oracle variants:

\begin{itemize}
    \item \textbf{Oracle-Easy:} selects samples with the lowest prediction loss, corresponding to the easiest transitions.
    \item \textbf{Oracle-Uniform:} partitions samples into disjoint prediction-loss intervals and selects high-loss samples within each interval, ensuring balanced coverage across different difficulty levels.
\end{itemize}

To analyze the behavior of different exploration strategies, we visualize the distribution of prediction errors—measured by the world model’s prediction loss—over the samples selected by each method (\cref{fig:al_distribution}).
This analysis reveals how different selection criteria bias the collected data toward specific difficulty regimes and provides insight into their impact on world model learning.

\textbf{Qualitative Results.}
\label{sec:Qualitative Results of MiniGrid}
\cref{fig:qual_part1} presents a qualitative comparison of world model predictions across interactive actions.
While most methods perform similarly on simple motion-dominated transitions, clear differences arise for interaction-centric actions such as \textit{Toggle} and \textit{Swap}.
In these cases, models trained with actively selected data more accurately capture interaction-induced state changes, whereas \textit{Random} exhibit a strong bias toward predicting frequent but uninformative movement actions (e.g., \textit{Turn}), highlighting the benefit of informative data selection under distribution shift.

\begin{figure}[t]
    \centering
    \includegraphics[width=0.4\linewidth]{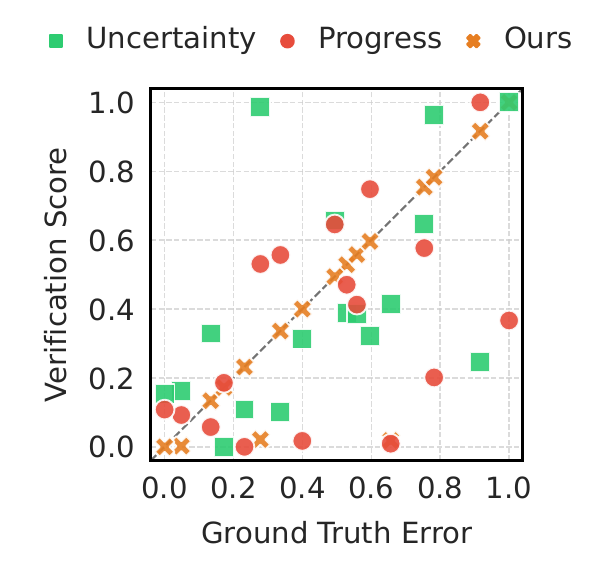}
     \caption{Verification score vs. ground-truth error. Our method exhibits a strong monotonic alignment with true error, whereas baselines show scattered distributions and frequent misranking.}
    \label{fig:scatter}
\end{figure}

\textbf{Verification Score vs. Ground Truth Error.}
As shown in Figure~\ref{fig:scatter}, our method exhibits a strong monotonic relationship between verification score and true error, with samples distributed closely along the diagonal. This indicates that the proposed verification mechanism provides a faithful estimate of sample difficulty, assigning higher scores to transitions that are indeed harder for the world model.

In contrast, baseline methods such as \textit{Uncertainty} and \textit{Progress} display significantly weaker alignment, with scattered distributions and frequent misranking of samples. In particular, they tend to assign high scores to samples with relatively low true error or fail to consistently identify high-error transitions.

These observations are consistent with the quantitative results reported in Figure~\ref{fig:minigrid_results_2} (Mid), where our method achieves the highest Spearman and Kendall correlations. Together, the results further validate that our verification scores accurately capture the underlying difficulty of transitions, leading to more effective data selection.

\begin{figure}
    \centering
    \includegraphics[width=0.8\linewidth]{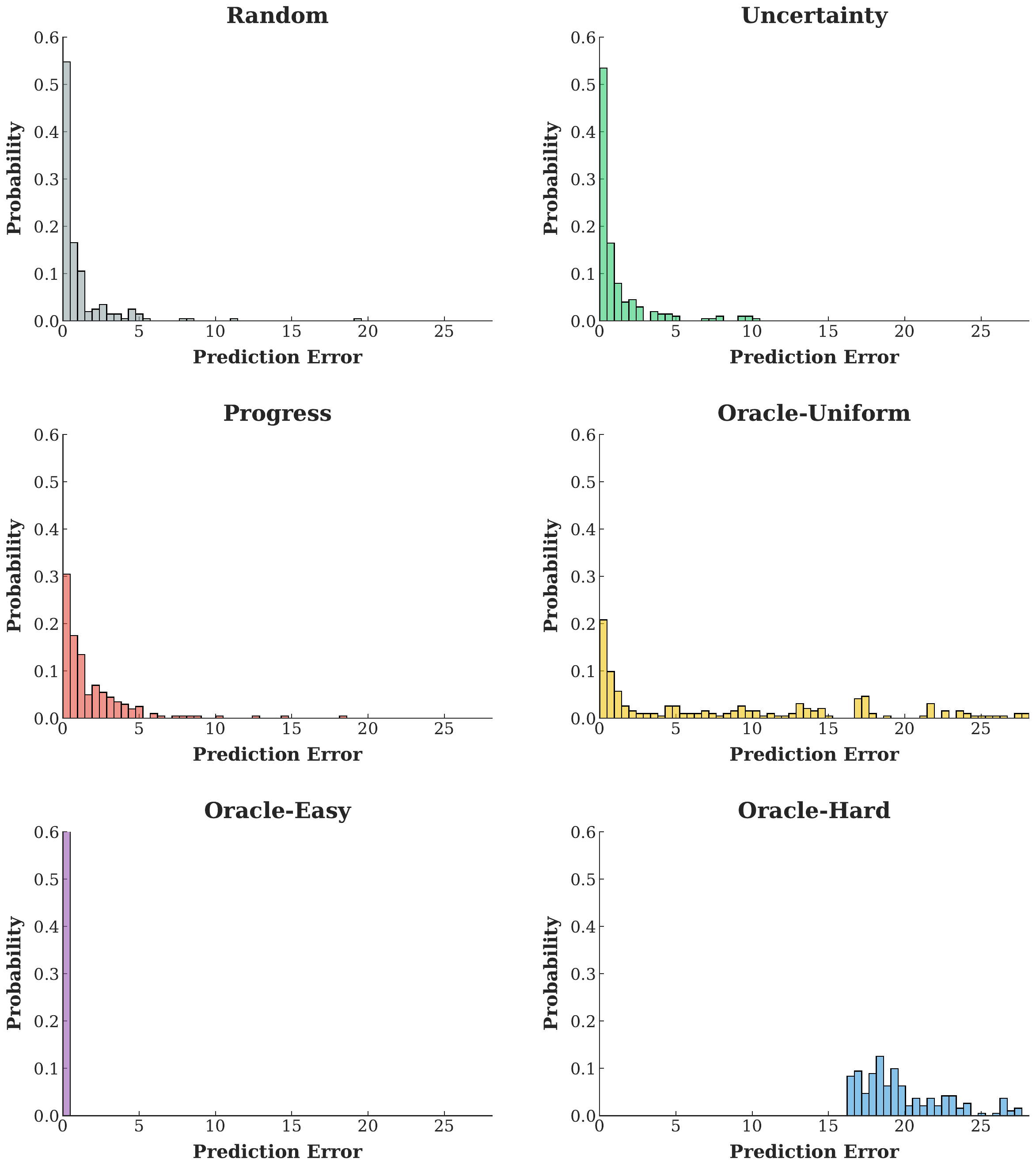}
    \caption{The distribution of world model's prediction error on the data selected by different exploration methods in MiniGrid.}
    \label{fig:al_distribution}
\end{figure}

\ifnips
\begin{figure}[t]
    \centering
    \includegraphics[width=\linewidth]{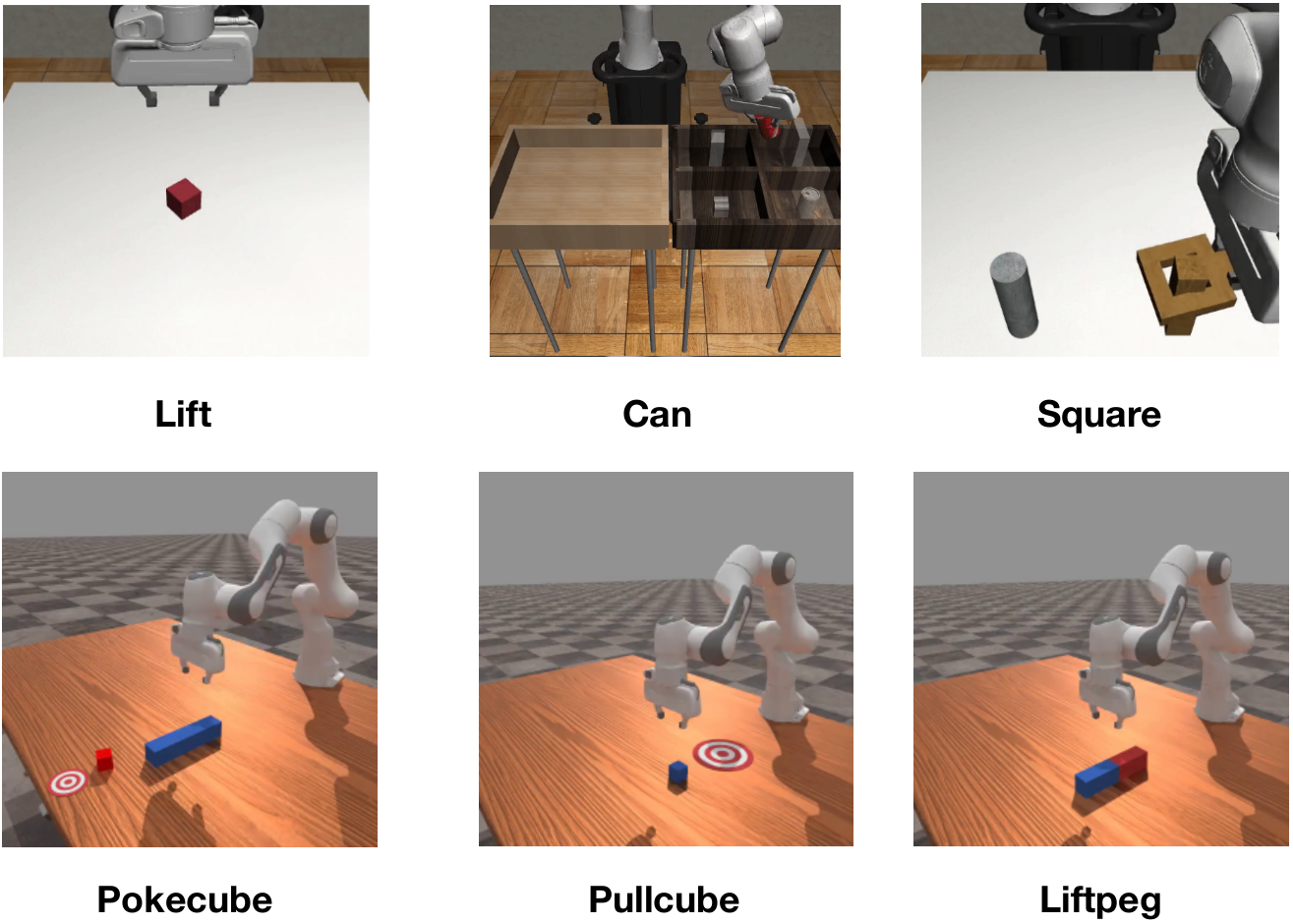}
    \caption{Visual description of all tasks used from RoboMimic (Row 1) and Maniskill (Row 2).}
    \label{fig:intro_env}
\end{figure}
\else
\begin{figure}
    \centering
    \includegraphics[width=\linewidth]{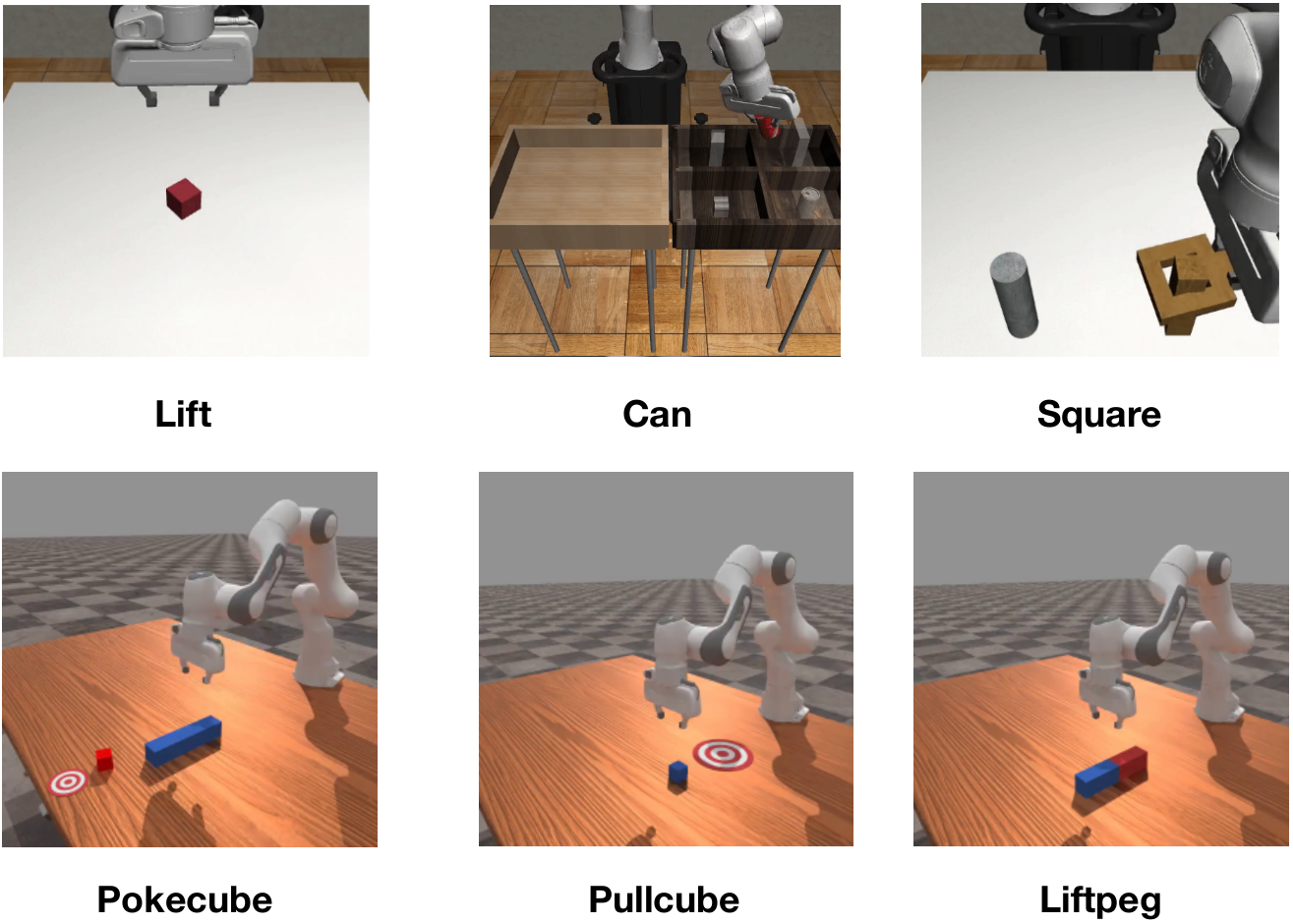}
    \caption{Visual description of all tasks used from RoboMimic (Row 1) and Maniskill (Row 2).}
    \label{fig:intro_env}
\end{figure}
\fi

\subsection{Robotic Domains Setting}
The experiments are conducted on simulated robotic manipulation tasks from Robomimic~\citep{zhu2020robosuite} (\texttt{Lift}, \texttt{Can}, \texttt{Square}) and ManiSkill~\citep{mu2maniskill} (\texttt{PullCube}, \texttt{PokeCube}, \texttt{LiftPeg}). Figure~\ref{fig:intro_env} provides a visualization of these tasks. For each task, the agent observes RGB images from both a wrist-mounted camera and a front-facing camera. In addition, proprioceptive observations are provided, including the end-effector position and orientation, as well as the gripper position. The action space is a 7-dimensional continuous vector in $[-1,1]$, including the control changes in the end-effector position and orientation, and the opening and closing states of the gripper.

\subsubsection{Additional Details on Setups.}
\textbf{Dataset Collection.}
For both benchmarks, we train 10 diffusion policy models to collect a diverse set of trajectories, resulting in a total of 1500 samples per task. We follow the default horizon settings of each environment: 100 steps for \texttt{Lift}, 200 for \texttt{Can} and \texttt{Square}, 50 for \texttt{PullCube}, 100 for \texttt{PokeCube}, and 150 for \texttt{LiftPeg}. We use nine of these checkpoints to construct the exploration dataset, and reserve the remaining checkpoint as a validation set for evaluating world model learning quality. 

\textbf{Reward Evaluation.} We follow the reward formulation of SAILOR~\citep{jain2025smoothseaskilledsailor}. Specifically, we train a reward model (using the same architecture and hyperparameters as in SAILOR) to score the latent states of our world model based on how expert-like they are. The reward model is trained as a discriminator between latent embeddings from expert rollouts and learner rollouts, using a moment-matching objective with a gradient penalty.

\subsubsection{Details on World Models.}
We adopt an action-conditioned world model based on Dreamer-V3~\citep{hafner2023mastering}.
Visual observations are encoded using a convolutional encoder with stride-2 convolutions, and proprioceptive states
are embedded with a 5-layer MLP. Based on empirical performance, image and state inputs
are processed by separate encoders. For decoding, image observations are reconstructed
using a transposed-convolutional decoder with stride-2 upsampling, and proprioceptive
states are reconstructed via a 5-layer MLP. 

The latent state $z_t$ comprises a deterministic recurrent component $h_t$ and a
stochastic component $s_t$. The deterministic state is modeled by a GRU with a
512-dimensional hidden state and is updated using the previous latent state $z_{t-1}$
and action $a_{t-1}$. The resulting recurrent state is then used by the dynamics model
to parameterize the distribution of the stochastic latent variable $s_t$. The number of dimensions of stochastic representation is $1024$. The number of the rollout horizon is $32$. Other training hyperparameters are given in Table~\ref{tab:wm_hyperparams}.  
\begin{table}[t]
\centering
\caption{World model training hyperparameters.}
\begin{tabular}{l|c}
\hline
\textbf{World Model} & \\ \hline
Replay capacity & $1 \times 10^{5}$ \\
Batch size & 16 \\
Batch length & 32 \\
Optimizer & Adam \\
Reconstruction loss scale & 1.0 \\
Learning rate & $1 \times 10^{-4}$ \\
\hline
\end{tabular}
\label{tab:wm_hyperparams}
\end{table}

\subsection{Details on Inverse Dynamic Models.}
\label{sec:idm_architecture}

We adopt the inverse dynamics modeling framework from CLAM~\citep{liang2025clam} as our base IDM architecture.
Given consecutive observations \((o_t, o_{t+1})\), a latent inverse dynamics model infers a continuous latent
action \(z_t = f_\phi(o_t, o_{t+1})\) that captures the underlying transition.
This latent action is then used to condition a latent forward dynamics model,
\(g_\psi(o_{t+1} \mid o_t, z_t)\), which predicts the next observation \(\hat{o}_{t+1}\). An action decoder \(p_\omega(a_t \mid z_t)\) maps the latent action back
to the environment action space.
The IDM, FDM, and action decoder are jointly trained using a combination of reconstruction and action
prediction losses over both labeled and unlabeled trajectories.
Compared to CLAM, we encourage sparsity in the latent action space by applying an
\(\ell_1\) regularization term to the inferred latent actions \(z_t\).
encouraging the model to discover structured and task-relevant factors. 

For visual observations, following CLAM, we adopt a space-time Transformer~\citep{bertasius2021space} for encoders. 
Each \(256 \times 256 \times 3\) RGB image is first partitioned into non-overlapping
\(16 \times 16\) patches, yielding 16 visual tokens per frame, which are projected
into a shared hidden space via a linear embedding layer.
The encoder is composed of the stacked ST attention layers. In the decoder, each ST block performs cross-attention between
visual tokens and the latent action representations produced by the
encoder. Other hyperparameters are given in Table~\ref{tab:idm_hyperparams}. 
\begin{table}[t]
\centering
\caption{Hyperparameters for inverse dynamics and action decoding.}
\begin{tabular}{l c}
\hline
\textbf{Hyperparameter} & \textbf{Value} \\
\hline
Num updates & 500{,}000 \\
Train action decoder every & 2 \\
Action decoder batch size & 128 \\
Action decoder loss weight & 1 \\
Action decoder hidden dim & [1024, 1024, 1024] \\
Action decoder embedding dim & 512 \\
Reconstruction loss weight & 1 \\
Sparsity loss weight & 0.1 \\
Latent action dim & 16 \\
Context len & 2 \\
Embedding dim & 128 \\
\hline
\end{tabular}
\label{tab:idm_hyperparams}
\end{table}

\subsubsection{Visualization}
Fig.~\ref{fig:vis-lift}–\ref{fig:vis-square} visualize open-loop rollouts on Robomimic-Lift and Robomimic-Square. Overall, the base model exhibits poor visual predictions, with noticeable degradation in both rendering quality and dynamical consistency. Incorporating uncertainty- and progress-aware exploration substantially improves visual fidelity with more samples, producing sharper and more coherent renderings over time. In contrast, the vanilla IDM improves the accuracy of the underlying dynamics, indicating that inverse dynamics supervision helps recover action-relevant transitions. However, over long horizons, particularly in the final one to two frames, noticeable discrepancies remain, such as misaligned gripper orientations and inaccurate object occlusions. Our sparse IDMs further mitigate these long-horizon errors, yielding more stable dynamics and better-preserved fine-grained details in the predicted rollouts.

\section{Additional Theoretical Derivation}

\subsection{Detailed Derivation for \cref{subsec:identifiability}}
\label{app:identifiability}

This appendix provides a compact, self-contained proof of the self-improvement guarantee stated in
\cref{subsec:identifiability}. The key idea is that, under a \emph{generation--verification gap}, the full transition
$(\zz^{t},\aa^{t}) \mapsto \zz^{t+1}$ may be out-of-support, while a small \emph{verification subset}
of latent variables remains on-support and is sufficient to recover the action. \methodacro uses this subset
to \emph{verify} (infer) missing action labels on OOS transitions, thereby expanding its action-labeled
support.

\subsubsection{Time-Lagged Latent Causal Model (TLCM)}
We consider a \emph{time-lagged latent causal model} (TLCM) with $k$ latent blocks
$\zz^t := (\zz_1^t,\dots,\zz_k^t)$, actions $\aa^{t}$, and observations $\xx^t$.
For clarity, we present the (deterministic) Markovian case used in our analysis:
\begin{align}
\zz^{t+1} &= g(\zz^{t}, \aa^{t}),\\
\xx^{t} &= \varphi(\zz^t),
\end{align}
where $\varphi$ is a diffeomorphism onto its image (so $\varphi^{-1}$ is well-defined on observed states).
The induced causal graph over $(\zz^{t},\aa^{t},\zz^{t+1})$ factorizes as
\begin{equation}
p(\zz^{t+1} \mid \zz^{t}, \aa^{t}) \;=\; \prod_{i=1}^k p(\zz_i^{t+1} \mid \mathrm{Pa}(\zz_i^{t+1})).
\label{eq:tlcm_factor}
\end{equation}

\subsubsection{Support, compositional OOS, and the verification subset}
Let $\mathcal{D}_{\mathrm{act}}\subset\{(\xx^{t},\aa^{t},\xx^{t+1})\}$ be the action-labeled seed dataset inducing
a seed distribution $P_{\text{seed}}(\zz^{t},\aa^{t},\zz^{t+1})$. Write
\[
S_{\text{seed}} \;:=\; \supp\big(P_{\text{seed}}(\zz^{t},\aa^{t})\big)
\]
for the on-support set of state--action pairs.

\begin{definition}[On-support vs.\ out-of-support (OOS)]
A state--action pair $(\zz^{t},\aa^{t})$ is \emph{on-support} if it lies in $S_{\text{seed}}$ and is
\emph{out-of-support} (OOS) otherwise. For any index set $\cU\subseteq\{1,\dots,k\}$ we also define the
marginal support
\[
S^{\cU}_{\text{seed}} \;:=\; \supp\big(P_{\text{seed}}(\zz_{\cU}^{t},\aa^{t})\big),
\]
and say $(\zz_{\cU}^{t},\aa^{t})$ is on-support if it lies in $S^{\cU}_{\text{seed}}$.
\end{definition}

\begin{definition}[Compositional OOS transition]
A transition $(\zz^{t},\aa^{t},\zz^{t+1})$ is \emph{compositional OOS} (w.r.t.\ $P_{\text{seed}}$) if
$(\zz^{t},\aa^{t})\notin S_{\text{seed}}$ but there exists a subset of variables $\cU$ such that
$(\zz_{\cU}^{t},\aa^{t})\in S^{\cU}_{\text{seed}}$.
Intuitively, novelty comes from an unseen \emph{combination} of factors, while at least one
subset transition remains within the training support.
\end{definition}

We now formalize which subset can serve as a verifier.

\begin{definition}[Source (insulated) set]
Let $\cS_{\mathrm{src}}\subseteq\{1,\dots,k\}$ be a set of latent blocks that is \emph{causally insulated}
from its complement in the TLCM graph:
\[
\forall i\in \cS_{\mathrm{src}}, \qquad
\mathrm{Pa}(\zz_i^{t+1}) \;\subseteq\; \{\zz_j^{t}: j\in \cS_{\mathrm{src}}\}\,\cup\,\{\aa^{t}\}.
\]
Equivalently, there are no directed edges from $\zz_{\setminus \cS_{\mathrm{src}}}^{t}$ into
$\zz_{\cS_{\mathrm{src}}}^{t+1}$.
\end{definition}

\begin{definition}[Verification subset] \label{def:verification_subset}
Let $\cS_{\mathrm{act}} := \{ i : \aa^{t}\in \mathrm{Pa}(\zz_i^{t+1})\}$ denote the \emph{action-influenced}
latent blocks. We define the \emph{verification subset} as
\[
\cS \;:=\; \cS_{\mathrm{src}} \cap \cS_{\mathrm{act}}.
\]
The subset we require to remain on-support for verification is the
\emph{intersection} of (i) source/insulated variables and (ii) action-influenced variables.
\end{definition}

\begin{assumption}[Generation--verification gap (information asymmetry)]
\label{ass:gv_gap_appendix}
There exists a verification subset $\cS$ such that for every (potentially OOS) transition
$(\zz^{t},\aa^{t},\zz^{t+1})$ we wish to label, the restricted state--action pair remains on-support:
\[
(\zz_{\cS}^{t},\aa^{t}) \in S^{\cS}_{\text{seed}},
\]
even though $(\zz^{t},\aa^{t})$ may lie outside $S_{\text{seed}}$.
\end{assumption}

\subsubsection{Two identifiability ingredients from prior work}
Our proof uses (i) identifiability of the latent blocks (up to permutation / element-wise transforms)
and (ii) identifiability of the action from on-support subset transitions.

\begin{condition}[Identifiable latent blocks via mechanism sparsity]
\label{cond:pd_ident}
This condition is adapted from Proposition~7 (together with Assumption~5) of
\citet{lachapelle2024nonparametric}. Consider a TLCM whose observation model is a diffeomorphism (so $\varphi$
is invertible on its image) and whose transition model is Markov with respect to a sparse dependency
graph between $(\zz^{t},\aa^{t})$ and $\zz^{t+1}$ (as in \cref{eq:tlcm_factor}).
Assume we learn a second TLCM $(\hat{\varphi},\hat{g},\hat{G})$ that is $(\zz,\aa)$-consistent with the
ground-truth model in the sense of \citet{lachapelle2024nonparametric} and that the ground-truth graph
satisfies their \emph{graphical criterion} (Assumption~5).
Then the learned latent variables are identifiable up to a permutation and element-wise invertible
transformations (``complete disentanglement''), so we can treat the learned blocks as the true blocks
up to a fixed relabeling.
\end{condition}

\begin{condition}[Identifiable action from subset transitions]
\label{cond:la_ident}
This condition is adapted from Theorem~1 of \citet{lachapelle2025identifiability} by substituting
$\xx\equiv\zz_{\cS}^{t}$ and $\xx'\equiv\zz_{\cS}^{t+1}$.
Let $\cS$ be a fixed verification subset and define the subset dynamics
$g_{\cS}(\zz_{\cS}^{t},\aa^{t}) := [g(\zz^{t},\aa^{t})]_{\cS}$ (well-defined whenever
$\cS\subseteq \cS_{\mathrm{src}}$).
Assume the following hold on $S^{\cS}_{\text{seed}}$:
\begin{enumerate}
    \item \textbf{Continuity:} for each action value $\aa$, the map $\zz_{\cS}^{t}\mapsto g_{\cS}(\zz_{\cS}^{t},\aa)$ is continuous;
    \item \textbf{Injectivity:} for every $\zz_{\cS}^{t}$, $g_{\cS}(\zz_{\cS}^{t},\aa_1)=g_{\cS}(\zz_{\cS}^{t},\aa_2)$ implies $\aa_1=\aa_2$;
    \item \textbf{Connected conditional support:} for each $\aa$ in the action support, $\supp\big(P_{\text{seed}}(\zz_{\cS}^{t}\mid \aa)\big)$ is connected;
    \item \textbf{Support overlap:} for any $\aa_1,\aa_2$ in the action support,
    $\supp\big(P_{\text{seed}}(\zz_{\cS}^{t}\mid \aa_1)\big)\cap\supp\big(P_{\text{seed}}(\zz_{\cS}^{t}\mid \aa_2)\big)\neq\varnothing$.
\end{enumerate}
Then the action is identifiable from the subset transition $(\zz_{\cS}^{t},\zz_{\cS}^{t+1})$ up to a
fixed relabeling: there exists an injective map $v$ (independent of $\zz_{\cS}^{t}$) such that any
solution of the latent-action reconstruction problem in \citet{lachapelle2025identifiability} recovers
$v(\aa)$ deterministically from $(\zz_{\cS}^{t},\zz_{\cS}^{t+1})$.
In particular, when the learned action alphabet matches the true discrete action set, this corresponds
to a permutation of action labels.
\end{condition}

\subsubsection{Verified self-improvement}
We now state and prove the main appendix result.

\begin{theorem}[Identifiability of Self-Improvement]
\label{thm:master}
Let $\mathcal{D}_{\mathrm{act}}\subset\{(\xx^{t},\aa^{t},\xx^{t+1})\}$ be action-labeled transitions sampled from
$P_{\text{seed}}$, and let $(\xx^{*,t},\xx^{*,t+1})$ be an additional unlabeled transition sampled from some
test distribution $p_{test}$. Let $(\zz^{*,t},\aa^{*},\zz^{*,t+1})$ denote the corresponding latent
transition in the TLCM, i.e.\ $\zz^{*,t+1}=g(\zz^{*,t},\aa^*)$ and $\xx^{*,t+1}=\varphi(\zz^{*,t+1})$.
Assume:
\begin{enumerate}
    \item \textbf{Latent blocks are identified} up to a fixed permutation / element-wise transform
    (Condition~\ref{cond:pd_ident});
    \item \textbf{Generation--verification gap} holds for the verification subset $\cS$
    (Assumption~\ref{ass:gv_gap_appendix});
    \item \textbf{Action is identifiable from subset transitions} on $S^{\cS}_{\text{seed}}$
    (Condition~\ref{cond:la_ident}).
\end{enumerate}
Let $h_{\psi}:(\zz_{\cS}^{t},\zz_{\cS}^{t+1})\mapsto\aa^{t}$ be an inverse dynamics model trained
on the on-support subset transitions in $\mathcal{D}_{\mathrm{act}}$. Then the missing action label for the
unlabeled transition is uniquely determined by $(\zz_{\cS}^{*,t},\zz_{\cS}^{*,t+1})$, and
\[
\hat{\aa}^{*} \;:=\; h_{\psi}(\zz_{\cS}^{*,t},\zz_{\cS}^{*,t+1})
\]
recovers the true action (up to the fixed label relabeling in Condition~\ref{cond:la_ident}; with labeled
data this relabeling is resolved so that $\hat{\aa}^{*}=\aa^{*}$). Consequently, the tuple
$(\xx^{*,t},\hat{\aa}^{*},\xx^{*,t+1})$ is correctly labeled while it may satisfy
$(\zz^{*,t},\aa^{*})\notin S_{\text{seed}}$, i.e.\ it can expand the action-labeled support.
\end{theorem}

\begin{proof}
By Condition~\ref{cond:pd_ident} (a restatement of \citet[Prop.~7]{lachapelle2024nonparametric} in our
notation), the learned representation can be aligned with the ground-truth latent blocks up to a fixed
permutation and element-wise invertible transforms. This alignment preserves the block structure and
(up to a fixed relabeling) the parent/child relations in the TLCM graph, so the verification subset
$\cS=\cS_{\mathrm{src}}\cap\cS_{\mathrm{act}}$ is well-defined and accessible from observations via the
learned encoder.

By Assumption~\ref{ass:gv_gap_appendix}, the subset state--action pair
$(\zz_{\cS}^{*,t},\aa^*)$ lies in the on-support set $S^{\cS}_{\text{seed}}$, even if the full pair
$(\zz^{*,t},\aa^*)$ is OOS. Therefore, the on-support subset transitions contained in
$\mathcal{D}_{\mathrm{act}}$ are sufficient to train an inverse model $h_{\psi}$ for $g_{\cS}$.

Finally, by Condition~\ref{cond:la_ident} (adapted from \citet[Thm.~1]{lachapelle2025identifiability}
with $\xx\equiv\zz_{\cS}^{t}$ and $\xx'\equiv\zz_{\cS}^{t+1}$), the action is identifiable from the
subset transition $(\zz_{\cS}^{t},\zz_{\cS}^{t+1})$ up to a fixed relabeling. Hence applying
$h_{\psi}$ to the on-support subset transition $(\zz_{\cS}^{*,t},\zz_{\cS}^{*,t+1})$ recovers the
correct action label (after resolving the fixed relabeling using the labeled actions in
$\mathcal{D}_{\mathrm{act}}$). This yields the correctly labeled tuple
$(\xx^{*,t},\hat{\aa}^{*},\xx^{*,t+1})$, which can lie outside the original support
$S_{\text{seed}}$ and thus expands the action-labeled coverage.
\end{proof}

\subsubsection{Implications of the generation--verification gap}
\label{sec:practical_implications}

We now interpret Theorem~\ref{thm:master} and Assumption~\ref{ass:gv_gap_appendix} through a practical lens, identifying when \methodacro is most beneficial, when it degrades, and when it fails.

\paragraph{How large is the gap?}
\methodacro separates \emph{verification} (recovering the missing action) from \emph{generation} (predicting the full next state).
Verification only uses the subset transition on $\zz_{\cS}$, while generation must model the full $\zz$.
As a rule of thumb, \methodacro becomes most attractive when $\dim(\zz_{\cS}) \lll \dim(\zz)$: the inverse model stays simple and stable even as the world model faces increasingly many OOS compositions.

\paragraph{Condition 1: fixed verifier, growing scene (maximum benefit).}
If $\zz_{\cS}$ is agent-centric and fixed-dimensional (e.g., proprioception) while the rest of the scene grows in complexity (more objects, tools, contacts), then action recovery continues to rely on the same low-dimensional signal while the world model must extrapolate over a much larger state space.
In this regime, the benefit of \methodacro grows with scene complexity.
\emph{Example (warehouse robot).}
A 7-DoF manipulator arm has $\dim(\zz_{\cS}) = 7$ (joint angles/velocities).
In a warehouse with many objects, the world model must predict the state of each object (and their interactions), so $\dim(\zz)$ grows with scene complexity.
As scene complexity grows, predicting the full next state requires accounting for many interacting factors beyond the agent's direct control, which increases the difficulty of accurate forward prediction. In many control settings, however, the most useful signal is the action-imprinted change. This motivates focusing the inverse model on an agent-centric subset $\zz_{\cS}$, where action recovery can remain stable as the rest of the scene grows, yielding a practical forward–inverse gap that we exploit for self-improvement.

\paragraph{Condition 2: compositional OOS with preserved source-set (strong benefit).}
\methodacro succeeds when OOS novelty is concentrated in $\zz_{\setminus\cS}$ while the verifying subset behaves as it did in training.
Concretely, even when the full pair $(\zz^{t}, \aa^{t}) \notin S_{\text{seed}}$, the subset transition on $\zz_{\cS}$ remains on-support, so an inverse model trained on $\mathcal{D}_{\mathrm{act}}$ can still recover $\aa^{t}$ reliably.
\emph{Example (tool--object contact).}
Training contains ``move knife in free space'' and ``touch apple with hand,'' but not ``slice apple with knife.''
At test time, the full transition is OOS (contact dynamics are novel), yet the arm motion $\zz_{\cS}$ follows familiar trajectories. The inverse model can still recover the action, enabling the world model to learn the novel contact outcome.

\paragraph{Stochastic extension.}
For the remaining discussion we consider a stochastic generalization of the TLCM where $\zz^{t+1} = g(\zz^t, \aa^t, \bm{\epsilon}^t)$ with exogenous noise $\bm{\epsilon}^t$; Theorem~\ref{thm:master} holds in the deterministic special case $\bm{\epsilon}^t = \mathbf{0}$.

\paragraph{Condition 3: weak injectivity / action aliasing (degradation).}
Condition~\ref{cond:la_ident} (2) requires the mapping $\aa \mapsto g_{\cS}(\zz_{\cS}^{t}, \aa, \bm{\epsilon})$ to be injective.
When different actions produce indistinguishable (or nearly indistinguishable) subset transitions, recovered actions become ambiguous:
\begin{align}
\exists\, \aa \neq \aa' \;\text{s.t.}\; g_{\cS}(\zz_{\cS}^{t}, \aa, \bm{\epsilon}) \approx g_{\cS}(\zz_{\cS}^{t}, \aa', \bm{\epsilon}).
\end{align}
\emph{Example (underactuation / latency).}
In a soft gripper, different motor commands may produce nearly identical proprioceptive changes due to compliance.
Similarly, unmodeled communication delays can smear the action's effect across timesteps, violating injectivity.
In such cases, pseudo-labels drift and self-improvement degrades.

\paragraph{Condition 4: back-action from OOS into the verifier (failure).}
The verifying subset must remain insulated from OOS components (the \emph{source set} property of Definition~3).
\methodacro fails when OOS variables feed back into the verifier so that $\zz_{\cS}$ itself goes out-of-support.
One way to view the failure mode is that the verifier dynamics no longer depend only on $(\zz_{\cS}^{t}, \aa^{t})$, but also on $\zz_{\setminus\cS}^{t}$:
\begin{align}
\zz_{\cS}^{t+1} = g_{\cS}(\zz_{\cS}^{t}, \aa^{t}, \zz_{\setminus\cS}^{t}, \bm{\epsilon}^{t}).
\end{align}
\emph{Example (compliant contact).}
When a robot arm makes stiff contact with a deformable object, contact forces feed back into joint-level torques and velocities.
The ``verifying'' proprioceptive dynamics now depend on the OOS object state, breaking the source-set assumption and causing action recovery to fail.

\paragraph{Summary.}
\methodacro is most effective when (i) the verifying subset is small and fixed-dimensional relative to the full state, (ii) OOS novelty is confined to non-verifying blocks while $\zz_{\cS}$ stays on-support, and (iii) the action imprint on $\zz_{\cS}$ is strong (high injectivity).
It degrades under action aliasing and fails when OOS dynamics causally influence the verifier.

\subsection{Detailed Derivation for \cref{subsec:sample_complexity}}
\label{app:sample_complexity}

This appendix provides the exact derivation behind \cref{subsec:sample_complexity}. The purpose of the analysis is to isolate the statistical asymmetry exploited by \methodacro: predicting the full next state can be substantially harder than verifying the action from a low-dimensional action-relevant slice.

\paragraph{Setup.}
Let $s \in \R^{d_s}$ be the state, $a \in \R^{d_a}$ the action, and suppose the one-step dynamics are linear with additive Gaussian noise:
\begin{equation}
    s' = As + Ba + \xi,
    \qquad
    \xi \sim \mathcal{N}(0,\sigma_s^2 \Id_{d_s}).
    \label{eq:analysis_fwd_app}
\end{equation}
Assume also that there exists an action-relevant slice $z = Ms \in \R^{d_z}$, with $d_z \ll d_s$, from which the action can be linearly recovered up to irreducible ambiguity:
\begin{equation}
    a = H
    \begin{bmatrix}
    z \\
    z'
    \end{bmatrix}
    + \eta,
    \qquad
    z' = Ms',
    \qquad
    \eta \sim \mathcal{N}(0,\sigma_a^2 \Id_{d_a}).
    \label{eq:analysis_inv_app}
\end{equation}
We compare a dense forward regressor $\hat f$ trained on
\[
    x_F := \begin{bmatrix} s \\ a \end{bmatrix} \in \R^{d_s+d_a}
\]
and a sparse inverse regressor $\hat h$ trained on
\[
    x_I := \begin{bmatrix} z \\ z' \end{bmatrix} \in \R^{2d_z}.
\]
For analytic tractability, we assume both feature vectors have been whitened:
\begin{equation}
    x_F \sim \mathcal{N}(0,\Id_{d_s+d_a}),
    \qquad
    x_I \sim \mathcal{N}(0,\Id_{2d_z}),
    \label{eq:analysis_whitened_app}
\end{equation}
and both models are fit by ordinary least squares on $n$ i.i.d.\ labeled transitions.
(The whitening assumption simplifies the algebra; for general covariance $\Sigma$ the excess risk scales with $\mathrm{tr}(\Sigma^{-1})$---see, \eg, \citet{hsu2014random}---and the qualitative three-factor decomposition is preserved.)

As in the main text, we compare both models in the state space:
\begin{align}
    \mathcal{E}_F &:= \frac{1}{d_s}\,\E\!\left[\|\hat f(s,a)-f^\star(s,a)\|_2^2\right],
    \label{eq:analysis_ef_app} \\
    \mathcal{E}_I &:= \frac{1}{d_s}\,\E\!\left[\|f^\star(s,\hat h(z,z')) - f^\star(s,h(z,z'))\|_2^2\right].
    \label{eq:analysis_ei_app}
\end{align}

\begin{lemma}[OLS excess risk under isotropic Gaussian covariates]
\label{lem:ols_excess_app}
Consider scalar regression
\[
    y = \langle w^\star, x \rangle + \epsilon,
    \qquad
    x \sim \mathcal{N}(0,\Id_D),
    \qquad
    \epsilon \sim \mathcal{N}(0,\nu^2),
\]
with $n > D+1$. If $\hat w$ is the OLS estimator fit on $n$ i.i.d. samples, then its expected excess risk is
\begin{equation}
    \E\!\left[(\langle \hat w-w^\star, x\rangle)^2\right]
    = \nu^2\,\frac{D}{n-D-1}.
    \label{eq:ols_excess_app}
\end{equation}
This is a classical exact expression for well-specified linear regression with isotropic Gaussian design; see, \eg, \citet{hsu2014random,mourtada2022exact}.
\end{lemma}

\begin{proposition}[Exact forward--inverse gap in the linear--Gaussian model]
\label{prop:forward_inverse_gap_formal}
Under the setup above, let $\lambda := \|B\|_{\op}$.
If $n > d_s+d_a+1$ and $n > 2d_z+1$, then
\begin{align}
    \E[\mathcal{E}_F]
    &=
    \sigma_s^2\,\frac{d_s+d_a}{n-(d_s+d_a)-1},
    \label{eq:analysis_forward_exact_app} \\
    \E[\mathcal{E}_I]
    &\le
    \lambda^2\,\frac{d_a}{d_s}\,\sigma_a^2\,\frac{2d_z}{n-2d_z-1}.
    \label{eq:analysis_inverse_exact_app}
\end{align}
Consequently, the error ratio satisfies
\begin{equation}
    \Gamma(n)
    := \frac{\E[\mathcal{E}_F]}{\E[\mathcal{E}_I]}
    \;\ge\;
    \Bigl(\frac{d_s+d_a}{2d_z}\cdot\frac{d_s}{d_a}\Bigr)
    \cdot
    \Bigl(\frac{\sigma_s}{\lambda\,\sigma_a}\Bigr)^2
    \cdot
    \Bigl(\frac{n-2d_z-1}{n-(d_s+d_a)-1}\Bigr).
    \label{eq:analysis_ratio_exact_app}
\end{equation}
\end{proposition}

\begin{proof}
We apply \cref{lem:ols_excess_app} separately to the forward and inverse regressions.

\paragraph{Forward model.}
Each coordinate of $s'$ in \eqref{eq:analysis_fwd_app} is a scalar linear regression on the feature vector $x_F \in \R^{d_s+d_a}$ with noise variance $\sigma_s^2$. Therefore,
\[
    \E\!\left[(\hat f_j(s,a)-f^\star_j(s,a))^2\right]
    = \sigma_s^2\,\frac{d_s+d_a}{n-(d_s+d_a)-1}
\]
for each state coordinate $j \in \{1,\dots,d_s\}$. Averaging over the $d_s$ coordinates yields
\[
    \E[\mathcal{E}_F]
    =
    \sigma_s^2\,\frac{d_s+d_a}{n-(d_s+d_a)-1},
\]
which is exactly \eqref{eq:analysis_forward_exact_app}.

\paragraph{Inverse model in action space.}
Similarly, each coordinate of $a$ in \eqref{eq:analysis_inv_app} is a scalar linear regression on $x_I \in \R^{2d_z}$ with noise variance $\sigma_a^2$. Hence
\[
    \E\!\left[(\hat h_k(z,z')-h_k(z,z'))^2\right]
    = \sigma_a^2\,\frac{2d_z}{n-2d_z-1}
\]
for each action coordinate $k \in \{1,\dots,d_a\}$. Summing across the $d_a$ coordinates gives
\begin{equation}
    \E\!\left[\|\hat h(z,z')-h(z,z')\|_2^2\right]
    = d_a\,\sigma_a^2\,\frac{2d_z}{n-2d_z-1}.
    \label{eq:analysis_action_error_app}
\end{equation}

\paragraph{Mapping inverse error back to state space.}
Because the true dynamics $f^\star$ are linear in the action,
\[
    f^\star(s,\hat h(z,z')) - f^\star(s,h(z,z'))
    = B\bigl(\hat h(z,z')-h(z,z')\bigr).
\]
Therefore,
\begin{align}
    \mathcal{E}_I
    &= \frac{1}{d_s}\,\E\!\left[\left\|B\bigl(\hat h(z,z')-h(z,z')\bigr)\right\|_2^2\right] \\
    &\le \frac{\|B\|_{\op}^2}{d_s}\,\E\!\left[\|\hat h(z,z')-h(z,z')\|_2^2\right] \\
    &= \lambda^2\,\frac{d_a}{d_s}\,\sigma_a^2\,\frac{2d_z}{n-2d_z-1},
\end{align}
which proves \eqref{eq:analysis_inverse_exact_app} after taking expectations and using \eqref{eq:analysis_action_error_app}.

Finally, dividing \eqref{eq:analysis_forward_exact_app} by \eqref{eq:analysis_inverse_exact_app} yields \eqref{eq:analysis_ratio_exact_app}.
\end{proof}

\paragraph{Reading the bound.}
The ratio in \eqref{eq:analysis_ratio_exact_app} cleanly factorizes into three interpretable pieces.
The first term is a \emph{dimensionality advantage}: the forward model must estimate a map from $d_s+d_a$ inputs, whereas the sparse inverse model only uses $2d_z$ inputs.
The second term is a \emph{stochasticity advantage}: forward prediction suffers from environment noise $\sigma_s$, while inverse verification only suffers from the ambiguity of recovering the action from the selected slice, measured by $\sigma_a$, after accounting for the state-space gain $\lambda$.
The third term is a \emph{sample-size advantage}: when $n$ is only modestly larger than the dense forward dimension, the forward estimator is statistically much less stable.

\paragraph{Scope of the stylized model.}
This analysis is intentionally minimal. It does not claim that real robotic dynamics are linear or globally Gaussian. Instead, it isolates a statistical regime that matches the intuition behind \methodacro: if a low-dimensional subset preserves the action imprint while the full scene is high-dimensional, noisy, and sparsely labeled, then sparse inverse verification can be substantially more reliable than dense forward prediction.





\begin{figure*}[p]
    \centering
    \includegraphics[width=0.9\linewidth]{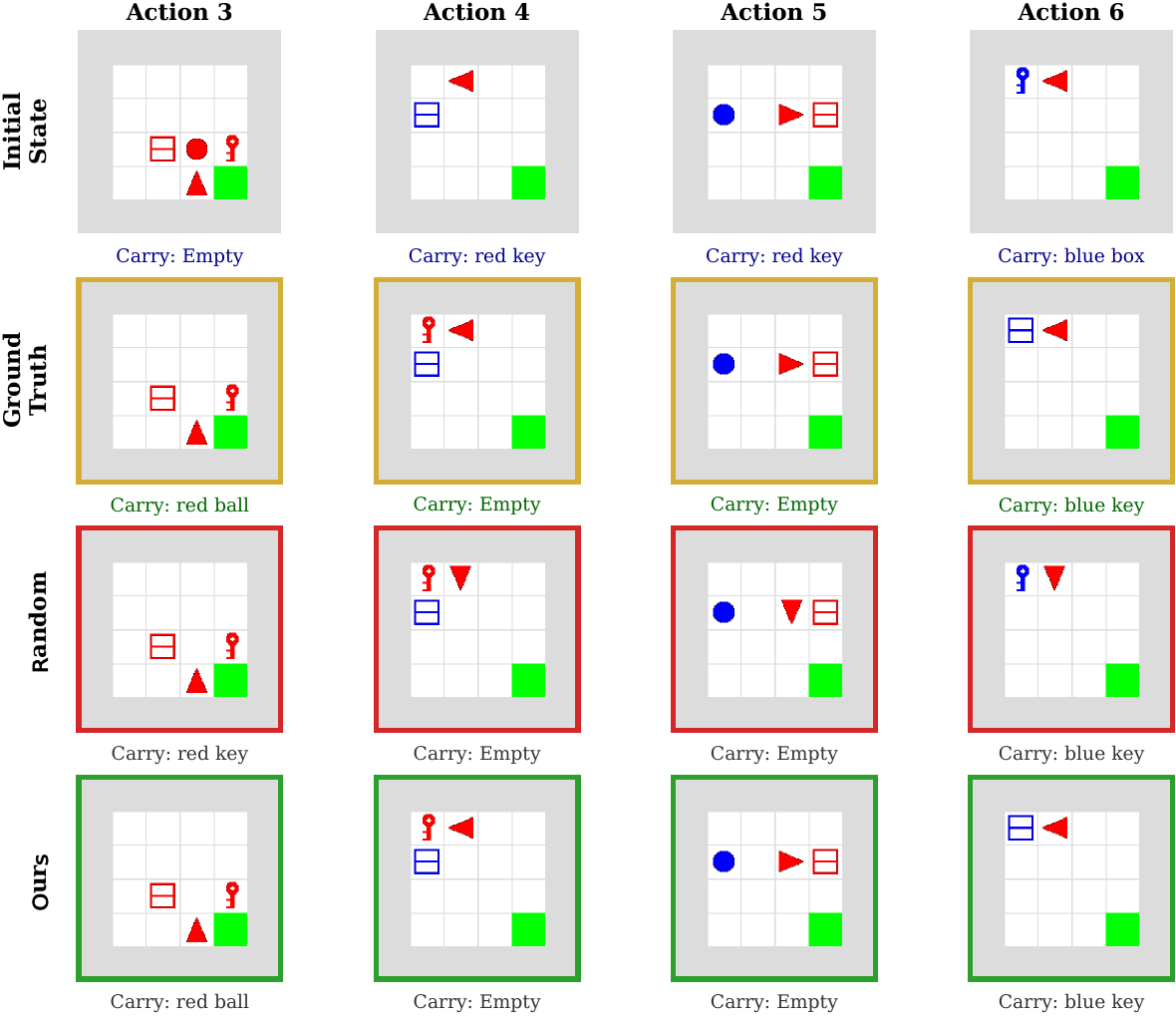}
    \caption{\textbf{Qualitative comparison of world-model rollouts under diverse interaction actions (Part I--II).}
    \textbf{Gold} borders denote ground-truth next observations; \textcolor[rgb]{0, 0.5, 0}{green} and \textcolor[rgb]{0.8, 0, 0}{red} borders indicate correct and incorrect predictions, respectively.
    Across both task sets, our method better preserves action-dependent state changes---notably for structured interactions such as \textit{Toggle} and \textit{Swap}---than exploration-based baselines.}
    \label{fig:qual_part1}
\end{figure*}

\begin{figure*}[p]
    \centering
    \ContinuedFloat
    \includegraphics[width=0.9\linewidth]{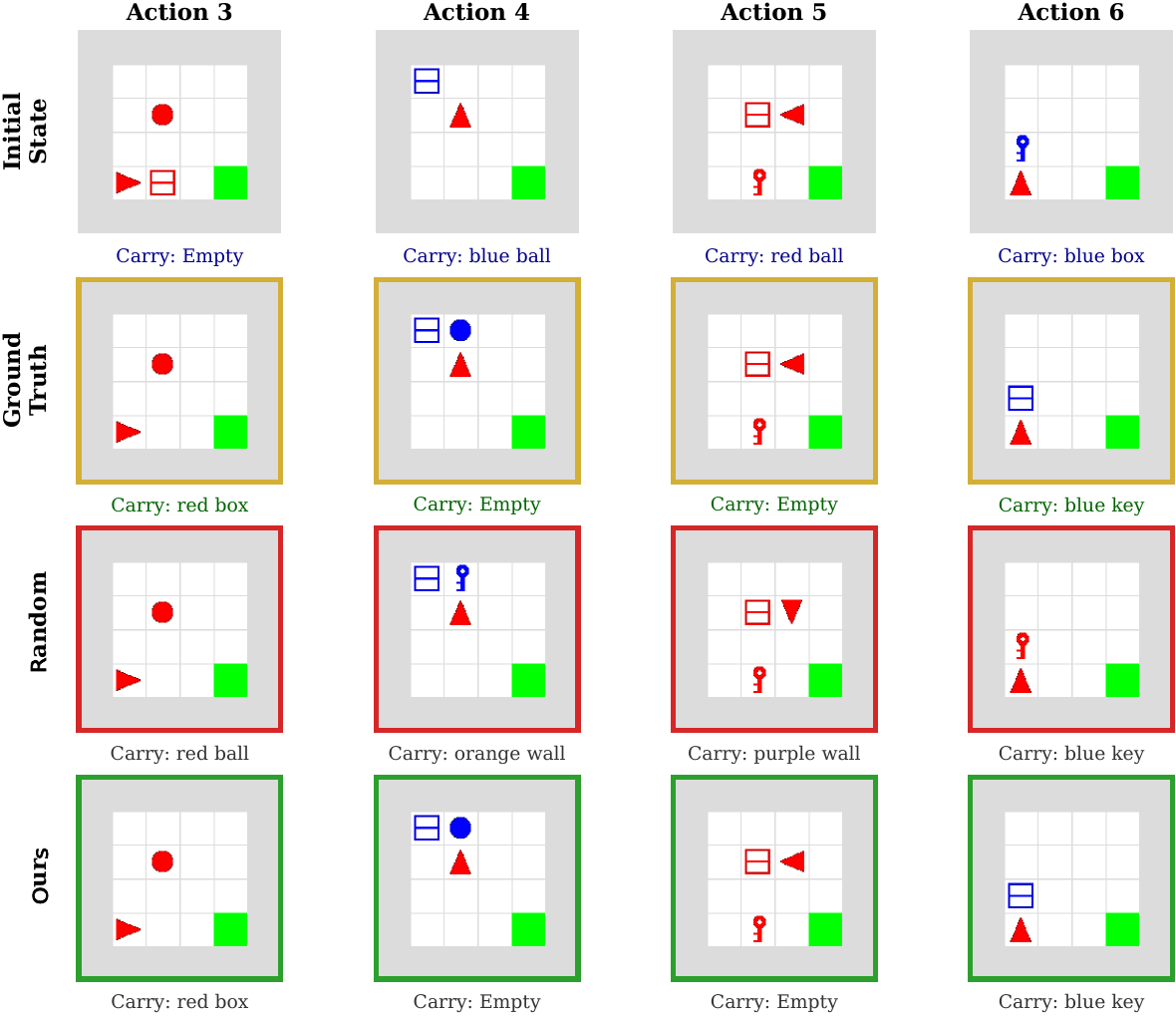}
    \caption{(Continued.) \textbf{Task Set B.} The \textbf{Random} baseline frequently collapses to predicting the most common primitive motions (e.g., \textit{Turn}), failing to model interaction-induced state changes (e.g., \textit{Toggle}). In contrast, models trained with data selected by our method capture these state transitions.}
    \label{fig:qual_part2}
\end{figure*}

\begin{figure*}[t]
    \centering
        \includegraphics[width=\linewidth]{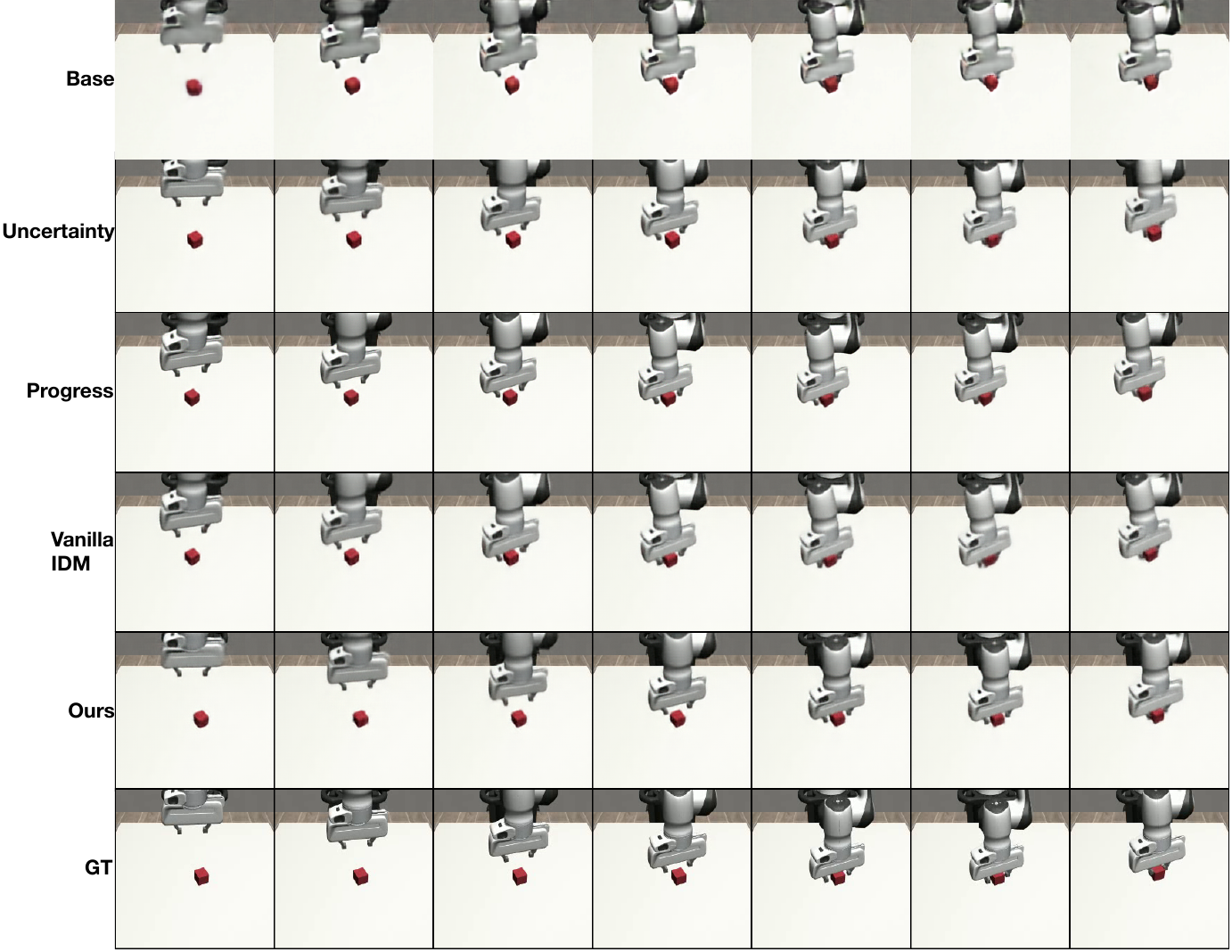}
        \label{fig:lift-vis}
    \caption{Qualitative comparison of world model predictions across different methods on Robomimic Lift.}
    \label{fig:vis-lift}
\end{figure*}

\begin{figure*}[t]
    \centering
        \includegraphics[width=1.0\linewidth]{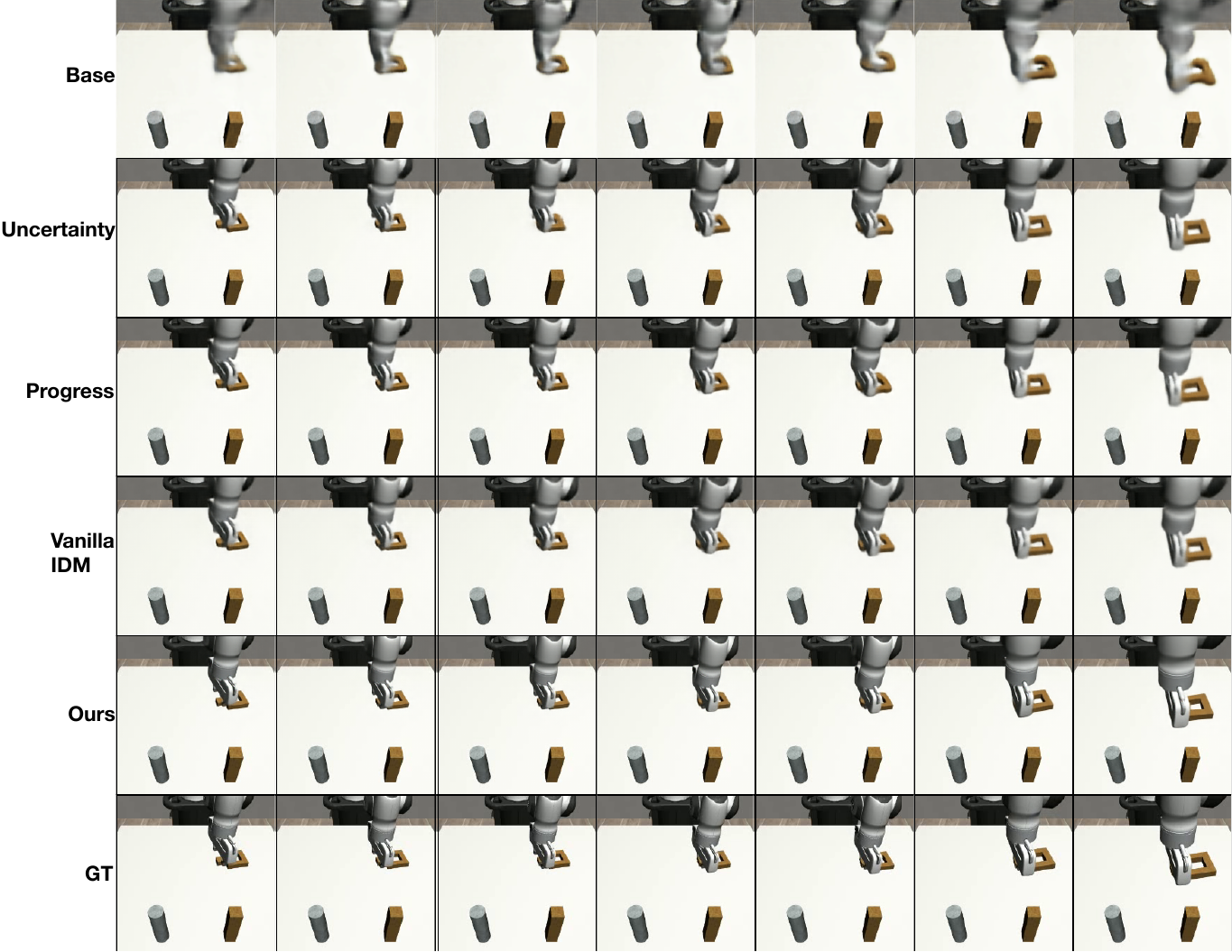}
        \label{fig:squre-vis}
    \caption{Qualitative comparison of world model predictions across different methods on Robomimic Square.} \label{fig:vis-square}
\end{figure*}

\section{Broader Impact}
\label{app:broader_impact}

\methodacro aims to improve the data efficiency and reliability of world-model learning for embodied agents. Potential benefits include reducing the amount of costly robot interaction needed for model improvement and making learned simulators more useful for policy evaluation. Potential risks include over-reliance on imperfect learned world models in safety-sensitive robotics settings, especially when verifier assumptions fail under domain shift. Any real-world deployment should therefore include task-specific validation, monitoring, and human oversight.

\clearpage
\newpage

\end{document}